\documentclass[lettersize,journal]{IEEEtran}
\usepackage{amsmath,amsfonts}
\usepackage{algorithmic}
\usepackage{algorithm}
\usepackage{array}
\usepackage[caption=false,font=normalsize,labelfont=sf,textfont=sf]{subfig}
\usepackage{textcomp}
\usepackage{stfloats}
\usepackage{url}
\usepackage{amsthm}
\usepackage{algorithm}
\usepackage{algorithmic}
\usepackage{verbatim}
\usepackage{graphicx}
\usepackage{cite}
\usepackage{wrapfig}
\usepackage{booktabs}
\usepackage{multirow}
\usepackage{orcidlink}
\newtheorem{lemma}{Lemma}
\newtheorem{theorem}{Theorem}
\usepackage{hyperref}
\usepackage{xcolor}
\usepackage{colortbl}
\begin{document}

\title{Learn from the Past: A Proxy Guided Adversarial Defense Framework with Self Distillation Regularization}

\author{Yaohua Liu, Jiaxin Gao, Xianghao Jiao, Zhu Liu, Xin Fan,~\IEEEmembership{Senior Member,~IEEE,} and Risheng Liu,~\IEEEmembership{Member,~IEEE}
        % <-this % stops a space
\thanks{This work was supported in part by the National Key R\&D Program of China under Grant 2022YFA1004101, in part by the National Natural Science Foundation of China under Grant U22B2052 and Grant 62027826, and in part by Liaoning Revitalization Talents Program under Grant 2022RG04. (Corresponding author: Risheng liu.)}
\thanks{Yaohua Liu, Jiaxin Gao, Xianghao Jiao and Zhu Liu are with the School of Software Technology, Dalian University of Technology, Dalian, Liaoning, 116024, China. (e-mail: liuyaohua\_918@163.com; jiaxinn.gao@outlook.com; jiaoxh0331@outlook.com; liuzhu@mail.dlut.edu.cn). }
\thanks{Xin Fan and Risheng Liu are with the School of Software Technology, Dalian University of Technology, Dalian, Liaoning, 116024, China. (e-mail: xin.fan@dlut.edu.cn; rsliu@dlut.edu.cn).}% <-this % stops a space
%\thanks{Manuscript received April 19, 2021; revised August 16, 2021.}}

}
\markboth{Journal of \LaTeX\ Class Files,~Vol.~14, No.~8, August~2021}%
{Shell \MakeLowercase{\textit{et al.}}: A Sample Article Using IEEEtran.cls for IEEE Journals}


\maketitle

\begin{abstract}
Adversarial Training (AT), pivotal in fortifying the robustness of deep learning models, is extensively adopted in practical applications. However, prevailing AT methods, relying on direct iterative updates for target model's defense, frequently encounter obstacles such as unstable training and catastrophic overfitting. In this context, our work illuminates the potential of leveraging the target model's historical states as a proxy to provide effective initialization and defense prior, which results in a general proxy guided defense framework, `LAST' ({\bf L}earn from the P{\bf ast}). Specifically, LAST derives response of the proxy model as dynamically learned fast weights, which continuously corrects the update direction of the target model. Besides, we introduce a self-distillation regularized defense objective, ingeniously designed to steer the proxy model's update trajectory without resorting to external teacher models, thereby ameliorating the impact of catastrophic overfitting on performance. Extensive experiments and ablation studies showcase the framework's efficacy in markedly improving model robustness (e.g., up to 9.2\% and 20.3\% enhancement in robust accuracy on CIFAR10 and CIFAR100 datasets, respectively) and training stability. These improvements are consistently observed across various model architectures, larger datasets, perturbation sizes, and attack modalities, affirming LAST's ability to consistently refine both single-step and multi-step AT strategies. The code will be available at~\url{https://github.com/callous-youth/LAST}.
\end{abstract}


\section{Introduction}

\IEEEPARstart{I}{n} the context of deep learning models and their widespread deployment in real-world applications~\cite{krizhevsky2012imagenet,jian2016deep,huang2023t}, there is a growing recognition of the susceptibility of these models to subtle adversarial perturbations in input data~\cite{kurakin2018adversarial,carlini2017adversarial,gu2023survey}. The introduction of perturbed adversarial samples can lead to the target model producing specified or alternative erroneous predictions, thus jeopardizing the functionality of real-world surveillance~\cite{dai2018cross}, autonomous driving systems~\cite{szegedy2013intriguing}, and giving rise to critical safety concerns. Consequently, the enhancement of model robustness against adversarial samples generated by various attacks has emerged as a focal research topic in the current landscape~\cite{papernot2016distillation,chen2020robust,latorre2023finding}.

While diversified defense techniques~\cite{zhang2019theoretically,dong2020adversarial} have been explored to mitigate adversarial attacks, Adversarial Training (AT)~\cite{madry2017towards,shafahi2019adversarial} is widely acknowledged as among the most efficacious strategies, of which the essence lies in addressing the min-max optimization problem. In line with the iterative format of adversarial attacks~\cite{rebuffi2022revisiting,yuan2021meta}, Standard AT (SAT) methods principally encompasses both single-step~\cite{goodfellow2014explaining} and multi-step~\cite{madry2017towards} training paradigms.  Broadly speaking, single-step AT, (e.g., Fast-AT~\cite{wong2020fast}), prioritizes enhancing computational efficiency to facilitate rapid solution convergence, whereas multi-step AT (e.g., PGD-AT~\cite{pang2020bag}) is mainly dedicated to augmenting robustness in the context of larger datasets and increased training expenditure. On top of that, several lines of works have explored heuristic defense techniques to enhance the defense process~\cite{jia2022adversarial,xu2023exploring}, e.g., introducing specialized regularization terms~\cite{zhang2019theoretically} or additional robust teacher models~\cite{pang2020bag,dong2022label}. YOPO employs the Pontryagin's Maximum Principle (PMP) to decouple the adversary update and back-propagation processes in AT. Fast-AT-GA~\cite{andriushchenko2020understanding} introduces a new regularization term, which maximizes gradient alignment within the perturbation set, improving Fast-AT's robustness and reducing the performance gap with more complex adversarial training methods. Fast-BAT~\cite{zhang2022revisiting} refines the process of AT by employing a bi-level optimization strategy and implicit gradient computation to make the training more robust. KD-AT proposes mitigating this label noise by recalibrating the labels assigned to adversarial examples, using the predictive label distribution from an adversarially trained model.

To summarize, these methods essentially spared efforts to introduce additional prior knowledge or design new learning strategies to assist the current state of target model itself to defend the adversarial example with explicit or implicit cost. Whereas, existing methods based on SAT always suffer from the unstable convergence behavior of target model~\cite{dong2022label, nakkiran2021deep}. Meanwhile, catastrophic overfitting~\cite{li2020towards} problem also  limits the robustness improvement, which refers to significant performance decrease during the training process. The adversarial perturbations, crafted within the attack's inner maximization loop of SAT, leverage the gradient data of the target model's current parameters, enhancing their capacity to jeopardize the target model.  

\begin{figure*}[!t]
	%	\vspace{-0.8cm}
	\begin{center}
		%		\renewcommand\arraystretch{0.5}
		%		\begin{tabular}{@{\extracolsep{-0.3em}}c@{\extracolsep{-0.1em}}}
			\includegraphics[height=4.0cm,width=18.0cm,trim=10 20 0 20,clip]{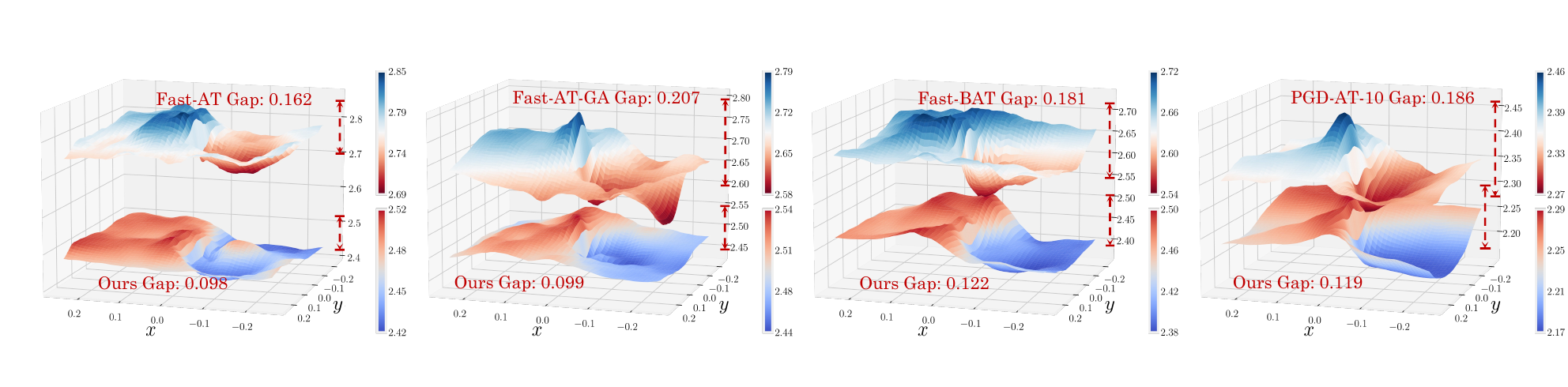}\\
			\footnotesize (a) Fast-AT\quad\quad\quad\quad\quad\quad\quad\quad\quad \footnotesize (b) Fast-AT-GA	\quad\quad\quad\quad\quad\quad\quad\quad\quad \footnotesize (c) Fast-BAT	\quad\quad\quad\quad\quad\quad\quad\quad \footnotesize (d) PGD-AT-10\quad \\ 
		\end{center}
		\caption{Comparison of the model's adversarial loss landscape trained by original SAT methods and their improved version of LAST. We also report the gap of maxi- and minimum losses for all the landscapes with $x, y\in[-0.25,0.25]$. Note that we have also labeled the loss range for the upper and lower surfaces on the left side of the axis. The models trained by LAST exhibit significantly lower loss, smoother loss landscapes along with smaller loss gaps.
		}\label{fig:loss_landscape}
	\end{figure*} 
	
In this work, we emphasize that during the defense process,  the historical parameter states of target model and their corresponding gradient directions remain concealed, thus serving as unknown elements to the adversarial examples and providing defense prior. From this new perspective, we establish a universal adversarial defense framework named `LAST' ({\bf L}earn from the P{\bf ast}), which ingeniously incorporates the model's historical states as the proxy model to fortify defense mechanisms. Specifically, we first generate response based on the proxy model to derive dynamically learned fast weights, which was used to calculate the differential unit as the corrected gradient direction for target model's subsequent update. Besides, we propose a self-distillation regularized defense objective that directs the update trajectory of the proxy model without necessitating additional teacher models and effectively alleviates the impacts of catastrophic overfitting on performance. Experimentally, across datasets of varying scales (i.e., CIFAR10, CIFAR100, and TinyImagenet dataset), and employing diverse network architectures (i.e., PreActResNet-18 and WideResNet-34-10), we have conducted comprehensive evaluations to substantiate the consistent and significant enhancement that the LAST framework brings to both training stability and ultimate robustness. These enhancements are evidenced through improvements in single-step and multi-step AT based on SAT formulation. 

\textbf{Visualization of Adversarial Landscapes.} Especially, in Fig.~\ref{fig:loss_landscape}, we compare the adversarial loss landscape~\cite{liu2020loss}  w.r.t. input variations between four SAT methods including Fast-AT, Fast-AT-FA, Fast-BAT, PGD-AT and their improved version under the LAST framework trained using PreActResNet-18 backbone with $\boldsymbol{\epsilon}=8/255$ on the CIFAR10 dataset. The adversarial loss is calculated with $\mathcal{L}_{\mathtt{atk}}(\boldsymbol{u} + x\boldsymbol{\vec{\iota}}+y\boldsymbol{\vec{o}})$, where $\boldsymbol{u}$ denotes the clean image, $\boldsymbol{\vec{\iota}}=\mathtt{sgn}(\nabla_{\mathbf{I}}\mathcal{L}_{\mathtt{atk}}(\boldsymbol{u})$ and $\boldsymbol{\vec{o}}\sim \textrm{Rademacher}(0,0.5)$ are the $sign$ gradient direction and random direction used to introduce the perturbation ($x$ and $y$ are the corresponding linear coefficients). We also calculated the loss gaps between the maxi- and minimum losses within the range of $x, y\in[-0.25,0.25]$. Typically speaking, these models with stronger defense capabilities are less sensitive to the adversarial example generated by add mixed perturbation along the gradient direction and random direction, thereby exhibiting lower overall loss value, relatively smoother loss landscapes along with smaller loss gap. \textbf{As it can be observed, the loss landscapes of models trained by the LAST framework show lower loss, smoother surfaces and smaller loss gaps, which exhibit enhanced robustness against both adversarial and random perturbations of different scales}. We summarize our main contributions as follows.

\begin{enumerate}
	\item \textit{New perspective.} Diverging from existing approaches that depend solely on the target model for adversarial defense, this study pioneers the integration of the historical states' response, which leverages the enhanced initial values and defensive prior to defend.
	
	\item \textit{General Defense Framework.} We  develop a universal adversarial defense framework, LAST, which incorporates the historical states of the target model as a proxy model and continuously adjust the update direction of the target model with dynamically learned fast weights. This framework consistently improves existing single-step and multi-step AT methods, offering an alternative to commonly used SAT pipeline.
	
	%	\item As one of the most significant features in this paper, we revisit the deficiencies of SAT from the perspective of its optimization trajectories, and introduce the historical state of the target model as its proxy model for gradient estimation. Then we construct a simple but much effective two-stage adversarial defense framework, which has great potential to serve as an alternative of SAT to improve the model robustness. 
	\item \textit{Self Distillation Defense.} Leveraging the introduced proxy model, we further design a self-distillation regularized defense objective which constrains the update direction of the proxy model without external teacher models, effectively mitigating the impact of catastrophic overfitting.
	
	\item \textit{Extensive Evaluation.} We implement the LAST framework based on various SAT methods, and verify its consistent performance improvement (e.g., up to $\bf 9.2\%$ and $\bf 20.3\%$ increase of RA compared with PGD-AT under AutoAttack ($\boldsymbol{\epsilon}=16/255$) on CIFAR10 and CIFAR100 datasets, respectively) with different backbones, datasets, attack modalities, and also demonstrate its ability to enhance training stability and ameliorate overfitting issues. 
\end{enumerate}

The subsequent content of this work is organized as follows. Sec.~\ref{sec:related_works} delves into adversarial attacks and defenses, along with a review of related work that leverages the historical states of parameters. In Sec.~\ref{sec:LAST}, we introduce the fundamentals of adversarial training, the motivation for referencing historical states, the construction of the LAST defense framework, and provide pertinent theoretical analyses and discussions. Sec.~\ref{sec:experiments} describes the implementation details such as experimental parameters, quantifies and visualizes the experimental results, and discusses the ablation study. Sect.~\ref{sec:conclusion} concludes this work.
	
\section{Related Works}\label{sec:related_works}

\subsection{Adversarial Attack} 

Generally speaking, two branches of adversarial attacks have been well explored including white-box and black-box attacks~\cite{rebuffi2022revisiting}. Black-box attacks refer to these attack strategies where the attacker has limited knowledge about the target model. Specifically, the attacker does not have access to the model's internal parameters, architecture, or training data. Here we focus on the white-box gradient-based adversarial attacks~\cite{yuan2021meta}, which possess full knowledge of the internal structure and parameters of the target deep learning model and leverage gradient information to craft adversarial samples. Specifically, single-step attack methods, e.g. Fast Gradient Sign Method (FGSM)~\cite{goodfellow2014explaining}, generates adversarial examples through a single, small perturbation of input data to produce erroneous classification results or misleading outputs in a single step. The single-step attack method could be naturally extended to the multi-step version by iterative optimization with small step size, e.g., BIM~\cite{kurakin2018adversarial} attack. Then PGD~\cite{madry2017towards} attack improves BIM attack with more attack steps and random initialization of the perturbation. Besides, C\&W attack generates adversarial examples that can deceive neural networks into misclassifying inputs with minimal perturbations. On top of that, AutoAttack~\cite{croce2020reliable} notably includes the Carlini \& Wagner (C\&W) attack and the DLR loss, a refined adversarial objective that enhances the attack's ability to exploit model vulnerabilities. By combining these sophisticated techniques, AutoAttack has been well recognized to offer a more robust and reliable benchmarking tool for assessing the defense mechanisms against adversarial threats. Under the proposed LAST framework, both single-step and multi-step attack strategies that can be used for AT are equally applicable and used to train more robust defense models.

\subsection{Adversarial Defense} 

Different branches of adversarial defense methods~\cite{zhang2019theoretically,dong2020adversarial} have been developed to enhance robustness of deep learning models against attacks, such as the preprocessing based methods~\cite{liao2018defense,jiao2023pearl} and training provably robust networks~\cite{dvijotham2018training,wong2018provable}. Among these defense methods, AT~\cite{madry2017towards} is widely recognized to be one of the most effective strategies. Based on the min-max formulation of SAT, single-step AT methods such as Fast-AT~\cite{wong2020fast} are proposed to implement computation-efficient fast training. Fast-AT-GA~\cite{andriushchenko2020understanding} adopts implicit GA regularization which yields better performance than Fast-AT. In recent works, Fast-BAT~\cite{zhang2022revisiting} incorporates the Implicit Gradient (IG) to estimate the hyper gradient based on the Bilevel Optimization (BLO) formulation and obtains the state-of-the-art performance. In addition, PGD based AT methods~\cite{pang2020bag,wang2019improving} have been continuously improved from different aspects by introducing heuristic techniques~\cite{zhang2020attacks,carlini2022certified} and prior knowledge~\cite{chen2020robust,dong2022label,latorre2023finding}. MART~\cite{wang2019improving} enhances adversarial robustness by distinctly treating misclassified and correctly classified examples during training, emphasizing the significant impact of the former on the final robustness. TRADES~\cite{zhang2019theoretically} is further proposed aimed for trade-off between robustness and natural accuracy. DyART~\cite{xu2023exploring} explores the decision boundary to mitigate conflicting decision boundary dynamics, leading to improved model robustness by directly manipulating training data margins different from SAT methods. Furthermore, RobustBench~\cite{croce2021robustbench} has been developed as a standardized benchmark to provide unified access to diversified robust models targeted at the image classification tasks. To improve the stability of training process as well as the final robustness performance, we establish the LAST framework by introducing history state prior for defense, which can consistently improve the above SAT-based defense methods.

In addition, several works also focus on the catastrophic overfitting problem during AT process and seek effective solutions. Typically, Rice et al.~\cite{rice2020overfitting} investigates overfitting in adversarially robust deep learning and demonstrates that early stopping during training mitigates this overfitting, matching or surpassing the performance improvements of more complex algorithmic advancements. However, this strategy has not entirely resolved the issue of crashes occurring during the training process. Besides, Kim et al.~\cite{kim2021understanding} proposes to prevent catastrophic overfitting during single-step adversarial training by adjusting the magnitude of perturbations for each image, thereby preserving the model's robustness against adversarial attacks without inducing decision boundary distortion. Dong et al.~\cite{dong2021exploring} explores the memorization effect in adversarial training (AT) and devise a new mitigation algorithm based on the analysis of memorization influence. Pang et al.~\cite{pang2022robustness} proposes to redefine robust error in adversarial training by replacing the inductive bias of local invariance with local equivariance. In contrast to the necessity of devising distinct algorithms or incorporating additional priors as suggested by the aforementioned approaches, the propose LAST framework alleviates the risk of catastrophic overfitting and enhances ultimate performance without introducing additional teacher models. This is accomplished by simply utilizing a new defense objective function based on the proxy model and the principle of self-distillation.

\subsection{Learning with Historical Parameter States} 

In addition, a few works have also explored the prior knowledge hidden in the historical weights of deep learning models to facilitate the training process, such as designing retrospective loss~\cite{jandial2020retrospective} or continuously averaging the historical weights, known as the Stochastic Weights Averaging (SWA) technique~\cite{izmailov2018averaging}, which have been applied in different fields~\cite{gowal2020uncovering}. Fast-SWA~\cite{athiwaratkun2018there} averages the weights of the network across different training epochs to achieve significantly improved generalization performance on semi-supervised learning tasks. Besides, SWALP~\cite{yang2019swalp} is proposed to average the model weights at low precision training environments, which has been shown to reduce the quantization noise and enhance the performance of neural networks. In comparison with these techniques applying averaged historical weights for different learning tasks, the proxy model under our LAST framework is introduced to dynamically generates fast responses to defend against attacks tailored for the current state of target model.

\section{Enhance Robustness with the LAST Framework}~\label{sec:LAST}
In this section, we first revisit the general min-max formulation of the SAT process. Then we clarify our motivation with detailed analysis of the input gradients to unveil the defense prior information hidden in the historical states. Next, we further introduce proxy model which is used to generate dynamically learned fast weights for defense, leading to the general LAST framework. Finally, we provide the theoretical basis and more discussion to analyze the effectiveness of this proxy-guided update rule.

\begin{figure*}[!t]
	%			\vspace{-0.3cm}
	\begin{center}
		%		\begin{tabular}{@{\extracolsep{-0.1em}}c@{\extracolsep{0.5em}}c@{\extracolsep{-0.1em}}c@{\extracolsep{-0.1em}}}
			\includegraphics[height=5.4cm,width=17.cm,trim=0 0 0 0,clip]{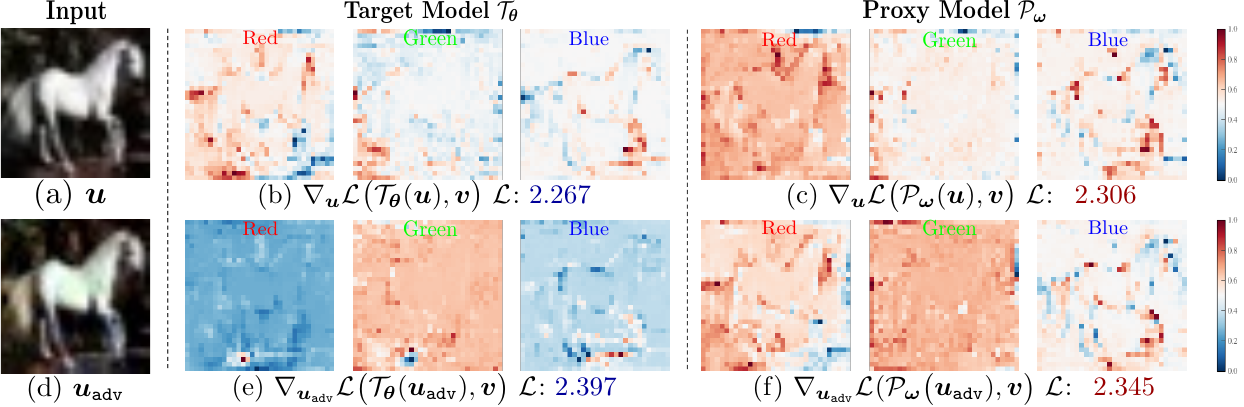}\\
	\end{center}
	\caption{Comparison of heat map of input gradient  w.r.t. the clean example $\boldsymbol{u}$ and adversarial example $\boldsymbol{u}_{\mathtt{adv}}$ between the target model $\mathcal{T}_{\boldsymbol{\theta}}$ and the introduced proxy model $\mathcal{P}_{\boldsymbol{\omega}}$. As it is shown, $\mathcal{P}_{\boldsymbol{\omega}}$ exhibits much less gradient variation in Red, Green and Blue channels.  Besides, it also has less growth of loss, i.e., (c)$\rightarrow$(f) and more salient input gradient w.r.t. $\boldsymbol{u}_{\mathtt{adv}}$ around the shape of the horse compared with $\mathcal{T}_{\boldsymbol{\theta}}$, i.e., (b)$\rightarrow$(e).}\label{fig:input_gradient}
\end{figure*}

\subsection{Preliminaries}
Generally speaking, SAT could be formulated as the min-max optimization problem~\cite{madry2017towards}, where the generated perturbation and defense model (i.e., target model) are alternatively optimized to improve the robustess of target model.  Here we first define the training dataset and input data pair  as $\mathcal{D}=\{\boldsymbol{u}_{i},\boldsymbol{v}_{i}\}_{i=1}^{\mathcal{M}}$, and denote the target model as $\mathcal{T}_{\boldsymbol{\theta}}$ parameterized by $\boldsymbol{\theta}$.  Then a general-purpose SAT formulation could be written as 
\begin{equation}
	\underset{\boldsymbol{\theta}}{\operatorname{min}} \text{ }\mathbb{E}_{\{\boldsymbol{u}_i, \boldsymbol{v}_i\} \in \mathcal{D}}\left[\underset{\boldsymbol{\delta}\in\mathcal{S}}{\operatorname{max}} \text{ }\mathcal{L}\bigl(\mathcal{T}_{\boldsymbol{\theta}}(\boldsymbol{u}_{i}+\boldsymbol{\delta}),\boldsymbol{v}_{i}\bigr)\right],
\end{equation} 
where $\boldsymbol{\delta}$ is the perturbation subject to $\mathcal{S}=\{\boldsymbol{\delta}|\|\boldsymbol{\delta}\|_{\boldsymbol{\rho}} \leq \boldsymbol{\epsilon}\}$ with $\boldsymbol{\epsilon}$-toleration $\boldsymbol{\rho}$ norm. The attack and defense objective, $\mathcal{L}_{\mathtt{atk}}$ and $\mathcal{L}_{\mathtt{def}}$ usually adopt the same form, denoted as $\mathcal{L}$. Typically, $\boldsymbol{\delta}$ is generated by $K$-step maximization of the attack objective following 
\begin{equation}
	\begin{aligned}
		\boldsymbol{\delta}_{k+1} \leftarrow &\mathtt{\Pi}_{\boldsymbol{\epsilon}}\bigl(\boldsymbol{\delta}_k + \boldsymbol{\alpha} \cdot \mathtt{sgn} \nabla_{\boldsymbol{\delta}} \mathcal{L}(\mathcal{T}_{\boldsymbol{\theta}}(\boldsymbol{u}_i+\boldsymbol{\delta}),\boldsymbol{v}_i)\bigr), \\
		&k =0,1,\cdots,K-1.
	\end{aligned}\label{eq:delta}
\end{equation}
where $\boldsymbol{\alpha}$ denotes the attack step size, $\mathtt{\Pi}$ and $\mathtt{sgn} $ are the projection and element-wise $sign$ operation. $\boldsymbol{\delta}_{0}$ is uniformly initialized from $(-\boldsymbol{\epsilon}, \boldsymbol{\epsilon})$. As for the  defense process, the target model performs gradient descent by
\begin{equation}
	\boldsymbol{\theta}_{i+1}=\boldsymbol{\theta}_{i}-\boldsymbol{\beta}\cdot\nabla_{\boldsymbol{\theta}} \mathcal{L}\bigl(\mathcal{T}_{\boldsymbol{\theta}}\left(\boldsymbol{u}_i+\boldsymbol{\delta}_{K}\right), \boldsymbol{v}_i\bigr).\label{eq:defense_theta}
\end{equation}
The above update rule of SAT could be described in the subfigure (a) of Fig.~\ref{fig:pipeline}. Following this update rule for SAT, different lines of works have been made to introduce extra prior or design new learning strategies along with explicit or implicit computation cost.

\subsection{Motivation}
In essence, existing methods all follow the standard update rule bv assisting the current state of target model $\boldsymbol{\theta}_{i}$ itself to defend the adversarial perturbation $\boldsymbol{\delta}_{K}$. Whereas, the perturbation is continuously optimized with full consideration of the current parameter state of the target model throughout the attack process of SAT, which leads to great vulnerability to the adversarial example. Whereas, \textbf{the historical states of the target model and its gradient information is inaccessible to the attack process, which is of great value to provide better initialization states and prediction prior information for defense process.}

To further validate this idea, we sample the clean example $\boldsymbol{u}$ from CIFAR10 dataset, and generate the adversarial example $\boldsymbol{u}_{\mathtt{adv}}$, i.e., $\boldsymbol{u}_{\mathtt{adv}}=\boldsymbol{u} + \boldsymbol{\delta}_{K}$ with Eq.~\eqref{eq:delta}, where $\boldsymbol{\delta}_{K}$ is iterative generated with PreActResNet-18 model under PGD-10 attack, $\boldsymbol{\epsilon}=8/255$.  We adopt the target model trained with early stopping, denoted as $\mathcal{T}_{\boldsymbol{\theta}}$, and define the historical state of $\mathcal{T}_{\boldsymbol{\theta}}$ as proxy model, written as $\mathcal{P}_{\boldsymbol{\omega}}$. Then $\boldsymbol{u}_{\mathtt{adv}}$ is used as input to attack both $\mathcal{T}_{\boldsymbol{\theta}}$ and $\mathcal{P}_{\boldsymbol{\omega}}$. In Fig.~\ref{fig:input_gradient}, we visualize the heatmap of input gradient to analyze how the target model and its proxy model react to $\boldsymbol{u}$ and $\boldsymbol{u}_{\mathtt{adv}}$ . Note that the three columns of subfigure (b), (c), (e) and (f) correspond to the input gradient w.r.t. Red, Green and Blue channels normalized to $[0,1]$.  

Indeed, the gradient w.r.t. different pixels of the input image essentially reflects how sensitive the model's prediction is to changes in each pixel value of this image~\cite{chan2020thinks}, and the output of robustly trained model will generate salient input gradients around the critical area of adversarial image which resembles the clean image. Generally speaking, $\mathcal{T}_{\boldsymbol{\theta}}$ obtained with early stopping has definitely stronger robustness than $\mathcal{P}_{\boldsymbol{\omega}}$. However, when faced with $\boldsymbol{u}_{\mathtt{adv}}$ generated with leaked gradient information of $\mathcal{T}_{\boldsymbol{\theta}}$, it can be observed from subfigure (b)$\rightarrow$(e) that the output of $\mathcal{T}_{\boldsymbol{\theta}}$ w.r.t.  $\boldsymbol{u}_{\mathtt{adv}}$ is seriously degraded and no longer produce salient input gradient around the horse in each channel. In comparison, $\mathcal{P}_{\boldsymbol{\omega}}$ is less sensitive to $\boldsymbol{\delta}_{K}$, and still has salient gradient around these pixels which matter most to the model decision in subfigure (c)$\rightarrow$(f). To summarize, $\mathcal{P}_{\boldsymbol{\omega}}$ exhibits stronger robustness against $\boldsymbol{\delta}_{K}$, which could serve as better initialization for defense compared with $\mathcal{T}_{\boldsymbol{\theta}}$ and leads to new defense mechanisms. 

\begin{figure*}[!t]
	%\vspace{-0.3cm}
	\begin{center}
			\includegraphics[height=4.5cm,width=15.9cm,trim=0 0 0 0,clip]{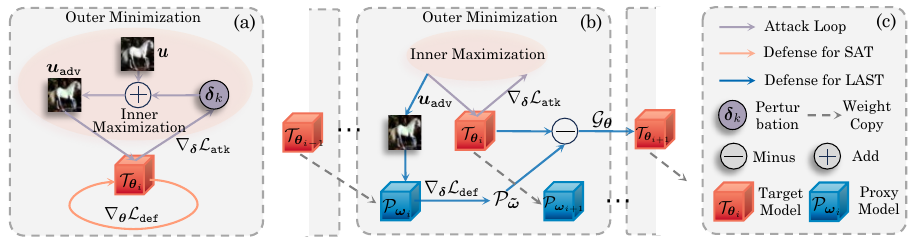}\\
	\end{center}
	\caption{Comparison of the attack and defense process between different paradigms. (a) SAT framework. (b) The LAST framework. (c) Description of the symbols. To avoid redundancy, the details of inner maximization process has been simplified in subfigure (b). 
	}\label{fig:pipeline}
\end{figure*}

\subsection{A Proxy Guided Adversarial Defense Framework}~\label{sec:method}

Based on the above analysis of $\mathcal{P}_{\boldsymbol{\omega}}$, we make the first attempt to introduce the proxy model to estimate better response to the generated attack w.r.t. the current state of target model. In detail, we first define the last state of target model as its proxy, i.e., $\boldsymbol{\omega}_{i}=\boldsymbol{\theta}_{i-1}, i=1,\cdots,\mathcal{M}$, where $\boldsymbol{\omega}_{0}$ is initialized using $\boldsymbol{\theta}_{0}$. In the following, we use $\mathcal{L}_{\mathtt{atk}}$ and $\mathcal{L}_{\mathtt{def}}$ to represent the attack and defense objective, respectively. Specifically, we adopt the same maximization process as SAT to generate $\boldsymbol{\delta}_{K}$ and devise a proxy guided update rule  for the defense strategy. At the first phase, we perform gradient descent with $\mathcal{P}_{\boldsymbol{\omega}}$ according to $\mathcal{L}_{\mathtt{def}}(\mathcal{P}_{\boldsymbol{\omega}}(\boldsymbol{u}_{i}+\boldsymbol{\delta}_{K}),\boldsymbol{v}_{i})$ to derive the fast wights $\boldsymbol{\tilde{\omega}}$ dynamically, which could be described as
\begin{equation}
	\begin{aligned}
		\boldsymbol{\tilde{\omega}}=\boldsymbol{\omega}_{i}-\boldsymbol{\beta}\cdot\nabla_{\boldsymbol{\omega}} \mathcal{L}_{\mathtt{def}}\bigl(\mathcal{P}_{\boldsymbol{\omega}_{i}}(\boldsymbol{u}_{i}+\boldsymbol{\delta}_{K}),~\boldsymbol{v}_{i}\bigr),\label{eq:omega_star}
	\end{aligned}
\end{equation}
where $\boldsymbol{\beta}$ denotes the learning rate of $\mathcal{P}_{\boldsymbol{\omega}}$. Note that Eq.~\eqref{eq:omega_star} is similar to the defense process of SAT in Eq.~\eqref{eq:defense_theta}, while $\boldsymbol{\delta}_{K}$ is tailored for $\mathcal{T}_{\boldsymbol{\theta}}$ instead of $\mathcal{P}_{\boldsymbol{\omega}}$. Then we employ the dynamically learned $\boldsymbol{\tilde{\omega}}$ and current state of target model (i.e., $\boldsymbol{\theta}_i$) to calculate the differential unit $\mathcal{G}_{\boldsymbol{\theta}}$, where $\mathcal{G}_{\boldsymbol{\theta}} = \boldsymbol{\theta}_{i} - \boldsymbol{\tilde{\omega}}$. For the second phase, we update $\boldsymbol{\omega}_{i}$ to record the current state of target model, and then perform gradient descent of $\boldsymbol{\theta}_i$ with the differential unit $\mathcal{G}_{\boldsymbol{\theta}}$. 

\begin{algorithm}[!t]
	\caption{The Proposed LAST  Framework }\label{alg:LAST}
	\begin{algorithmic}[1]
		%\REQUIRE Current UL variable $\x^t$ .
		%\ENSURE  $\y_K(\x^t,\z^t)$
		\REQUIRE Training epochs $\mathcal{J}$, $\mathcal{M}$ batches of data pairs ($\boldsymbol{u}_{i}$,~$\boldsymbol{v}_{i}$), attack iteration $K$, target model $\mathcal{T}_{\boldsymbol{\theta}}$,  and proxy model $\mathcal{P}_{\boldsymbol{\omega}}$.\ \ ($\boldsymbol{\alpha}$, $\boldsymbol{\beta}$, $\boldsymbol{\gamma}$ denotes attack step size,  the learning rate of $\mathcal{P}_{\boldsymbol{\omega}}$ and $\mathcal{T}_{\boldsymbol{\theta}}$.)
		%, attack step size $\alpha$, learning rate $\beta$ w.r.t. $p$ and $\gamma$ w.r.t. $\theta$.  
		\STATE \textit{// Initialize the proxy model $\mathcal{P}_{\boldsymbol{\omega}}$.}
		\STATE $\boldsymbol{\omega}_{0}=\boldsymbol{\theta}_{0}$.
		\FOR {$j=0 \rightarrow \mathcal{J}-1$}
		
		\FOR {$i=0 \rightarrow  \mathcal{M}-1$}
		\STATE Initialize $\mathbf{\delta}_{0}$.
		\STATE \textit{// Generate the perturbation with target model $\mathcal{T}_{\boldsymbol{\theta}}$.}
		\FOR {$k=0 \rightarrow K-1$}
		\STATE $\boldsymbol{\delta}_{k+1}=\boldsymbol{\delta}_{k}+\boldsymbol{\alpha} \cdot\mathtt{sgn}(\nabla_{\boldsymbol{\delta}}\mathcal{L}_{\mathtt{atk}}(\mathcal{T}_{\boldsymbol{\theta}_{i}}(\boldsymbol{u}_{i}+\boldsymbol{\delta}_{k}),\boldsymbol{v}_{i}))$. 
		\STATE $\boldsymbol{\delta}_{k+1}=\max \bigl[\min (\boldsymbol{\delta}_{k+1}, \boldsymbol{\epsilon}), -\boldsymbol{\epsilon} \bigr]$.
		\ENDFOR
		%		\STATE $g_{\theta}=0$
		%		\FOR {$k= 1\rightarrow  K$}
		\STATE \textit{// Phase 1: Estimate update direction of $\boldsymbol{\theta}_{i}$.}
		\STATE $\boldsymbol{\tilde{\omega}}=\boldsymbol{\omega}_{i}-\boldsymbol{\beta}\cdot\nabla_{\boldsymbol{\omega}} \mathcal{L}_{\mathtt{def}}\bigl(\mathcal{P}_{\boldsymbol{\omega}_{i}}(\boldsymbol{u}_{i}+\boldsymbol{\delta}_{K}),~\boldsymbol{v}_{i}\bigr)$. ~\label{alg:step_12}
		\STATE  $\mathcal{G}_{\boldsymbol{\theta}} = \boldsymbol{\theta}_{i} - \boldsymbol{\tilde{\omega}}$. \ \ (Compute the differential unit $\mathcal{G}_{\boldsymbol{\theta}}$)
		%		\ENDFOR
		\STATE \textit{// Phase 2: Update $\boldsymbol{\omega}_{i}$ and $\boldsymbol{\theta}_{i}$ sequentially.}
		\STATE $\boldsymbol{\omega}_{i+1}=\boldsymbol{\theta}_{i}$. 
		\STATE $\boldsymbol{\theta}_{i+1}=\boldsymbol{\theta}_{i} - \boldsymbol{\gamma}\cdot \mathcal{G}_{\boldsymbol{\theta}}$.~\label{alg:step_16}  \ \   
		\ENDFOR
		\ENDFOR
	\end{algorithmic}
\end{algorithm}

In Fig.~\ref{fig:pipeline}, we compare the defense processes of SAT with the proxy guided LAST framework. It is observed that the defense process of SAT, widely adopted by existing methods, updates the target model itself directly based on the generated adversarial examples, thereby weakening the defense capability due to the leakage of gradient information. In contrast, the defense process of the LAST framework is based on a proxy model without leakage of parameter states and gradient information, providing defensive prior knowledge and thereby enhancing robustness.
We also describe the whole algorithm of the LAST framework in Alg.~\ref{alg:LAST}. In the next subsection, we further introduce the self-distillation regularized defense objective to constrain the update of proxy model to estimate $\boldsymbol{\tilde{\omega}}$, which helps alleviate the catastrophic overfitting problem.

\subsection{Self Distillation Regularized Defense Objective}
Based on the introduced proxy model, which captures the historical states to introduce prior information for defense (Step~\ref{alg:step_12}-\ref{alg:step_16} in Alg.~\ref{alg:LAST}), we further delve into the defense objective to constrain the learning process of proxy model and alleviate the overfitting problem. Although $\mathcal{P}_{\boldsymbol{\omega}}$ is less sensitive to the adversarial attack targeted at $\boldsymbol{\theta}$, the perturbation still deteriorates the output of $\mathcal{P}_{\boldsymbol{\omega}}(\boldsymbol{u}_{\mathtt{adv}})$ leading to potential misclassification. Whereas, the direct output of target model, which refers to the soft targets in Knowledge Distillation (KD)~\cite{li2018exploring}, reflects which part the target model focuses on for prediction. 

When faced with the clean image and adversarial perturbation, the proxy model is supposed to generate outputs that have more similar distributions. Unlike these methods generating supervised soft targets with additional teacher models, we propose to constrain the estimation of $\boldsymbol{\tilde{\omega}}$ with the distance between soft targets of clean image and the corresponding adversarial image. Here we denote the temperature as $\boldsymbol{\tau}$, then the Self-Distillation (SD) defense objective could be written as follows	
\begin{equation}
	\begin{aligned}
\mathcal{L_{\mathtt{SD}}}= & \boldsymbol{\mu} \cdot \mathcal{L}_{\mathtt{KL}}\bigl(\mathcal{P}_{\boldsymbol{\omega}}(\boldsymbol{u}_{\mathtt{adv}}) / \boldsymbol{\tau}, \mathcal{P}_{\boldsymbol{\omega}}(\boldsymbol{u}) / \boldsymbol{\tau}\bigr)+  (1-\boldsymbol{\mu})\cdot  \\
&\mathcal{L}_{\mathtt{def}}\bigl(\mathcal{P}_{\boldsymbol{\omega}}(\boldsymbol{u}_{\mathtt{adv}}),\boldsymbol{v}\bigr) ,
	\end{aligned}
\end{equation}
where $\boldsymbol{\mu}\in[0,1)$ is the distillation coefficient to balance two loss terms, and $\mathcal{L}_{\mathtt{KL}}$ denotes the Kullback-Leibler (KL) divergence~\cite{kullback1951information} to measure the distance between two distributions of the soft targets. In this way, the proxy model is supposed to behave as consistently as possible when faced with clean or adversarial examples and generate correct and robust classification results meanwhile. The proposed SD loss, originated from the standard knowledge distillation loss with considerations for temperature and distillation coefficients, is restructured based on the introduced proxy model. Moreover, the defense objective supervises the learning process of proxy model without introducing (larger) pretrained teacher models or additional updates of models, thus can be flexibly integrated to the proposed algorithmic framework in Alg.~\ref{alg:LAST}. 

\subsection{Theoretical Basis and Discussion}\label{sec:discussion}	

\textbf{Theoretical basis.} We first provide the theoretical basis for this proxy guided update rule of LAST framework. 
\begin{lemma}\label{lm1}
	If $\mathcal{L}_\mathrm{def}$ is $L$-smooth, the updates to the proxy model $\tilde{\omega}$ are bounded, ensuring that the sequence $\{\theta_i\}$ exhibits stable convergence behavior.
\end{lemma}

\begin{proof}
	Assumed that $\mathcal{L}_\mathrm{def}$ is $L$-smooth with respect to $\omega$, then we have:
	\begin{equation}
		\|\nabla_\omega \mathcal{L}_{\mathrm{def}}(\omega_i)-\nabla_\omega \mathcal{L}_{\mathrm{def}}(\omega_{i-1})\|\leq L\|\omega_i-\omega_{i-1}\| \label{eq:lemma1}
	\end{equation}
	Using the update rule for $\tilde{\omega}$, the difference in consecutive updates is given by:
	$\|\tilde{\omega}_{i+1}-\tilde{\omega}_i\|\leq\beta L\|\omega_i-\omega_{i-1}\|$. This implies the updates to $\tilde{\omega}$ are bounded. Given $\gamma$ is small, we have 
		\begin{equation}
		\|\theta_{i+1}-\theta_i\|=\|-\gamma G_\theta\|=\gamma\|\tilde{\omega}-\theta_i\| \label{eq:lemma2}
	\end{equation}
	this ensures the stability of the update sequence $\left\{\theta_i\right\}.$
\end{proof}
\begin{theorem}\label{Thm1}
	Under the update rule given by LAST from Step 12-16 in Alg.~\ref{alg:LAST}, the sequence ot $\mathcal{T}_{\boldsymbol{\theta}}$ ,i.e., $\{\theta_i\}$ converges to a stable point, hence enhancing the model's robustness against adversarial attacks.
\end{theorem}

\begin{proof}
	To demonstrate convergence, we consider the sequence $\{\theta_i\}$ and show that it satisfies the conditions for a Cauchy sequence under the proposed update rule. Specifically, for any $\epsilon>0$, there exists an $N$ such that for all $m,n>N,\|\theta_m-\theta_n\|<\epsilon.$ Given the bounded updates and the diminishing nature of $G_{\theta}$ as $\theta_{i}$ approaches $\tilde{\omega}$, the sequence $\{\theta_i\}$ converges.
\end{proof}

\textbf{Further discussion.} Here we provide more discussion and different perspectives to analyze the effectiveness of the proxy guided  update rule. It can be observed from Eq.~\eqref{eq:lemma2} that the historical sequences of $\boldsymbol{\theta}_{i}$ is always constrained by the estimation of distance between $\boldsymbol{\tilde{\omega}}$ and $\boldsymbol{\theta}_{i}$, both of which is derived from $\boldsymbol{\theta}_{i-1}$. Therefore, update rule improves the consistency between adjacent states of $\boldsymbol{\theta}_{i}$, i.e., $\{\cdots,\boldsymbol{\theta}_{i}-\boldsymbol{\theta}_{i-1},\boldsymbol{\theta}_{i+1}-\boldsymbol{\theta}_{i},\cdots\}$. 

Besides, we could describe the update format as $\boldsymbol{\theta}_{i+1}=(1-\boldsymbol{\gamma})\cdot \boldsymbol{\theta}_{i}+\boldsymbol{\gamma}\cdot \boldsymbol{\tilde{\omega}}$. On top of that, $\boldsymbol{\gamma}$ serves as the aggregation coefficient to balance the influence of responses to historical attacks and current attacks. From this perspective, the format of this update rule is similar to these momentum based optimizers~\cite{sutskever2013importance} to some extent, which refer to the accumulation of historical gradients to perform gradient descent. Besides, we could also find evidence to demonstrate the effectiveness of proxy model from other general techniques such as SWA technique~\cite{izmailov2018averaging}, which was applied to smooth the weights by averaging multiple checkpoints along the training process. Whereas, SWA simply accumulates the exponential weighted average of the historical weights and target model's response. To summarize, the new update rule combines the response of proxy model, which bridges the historical states and current states to improve consistency among the optimization trajectories and introduce extra prior for the defense process.

\section{Experiments}~\label{sec:experiments}
In this section, we demonstrate the effectiveness of LAST framework by improving  various single-step and multi-step AT methods including SOTA methods under different settings. Besides, we also conduct the ablation experiments and validates its generalization performance under transfer-based black box attacks and out-of-distribution samples. All the experiments are conducted on a PC with 32GB RAM and 24GB NVIDIA RTX-4090 GPU.

\subsection{Implementation Details}

        \begin{table*}[!htbp]
	\centering
	%	\vspace{-0.3cm}
	\caption{We report the SA and RA of Fast-AT, Fast-AT-GA and Fast-BAT under PGD attack (PGD-10 and PGD-50) and AutoAttack. We use $m${\footnotesize$\pm n$} to denote the mean SA (i.e., $m$) with standard deviation (i.e., {\footnotesize$n$}) by running all the algorithms with 3 random seeds. Besides, we also present the runtime for three algorithms, both before and after improvement, to evaluate the computational cost of our defense framework.}
	\label{tab:AT_results1}
	\renewcommand\arraystretch{1.4}
	\setlength{\tabcolsep}{1.2mm}{
		\begin{tabular}{c|c|c|c|c|c|c|c|c}
			\toprule \hline \multicolumn{9}{c}{ CIFAR-10 dataset, PARN-18 backbone trained with $\boldsymbol{\epsilon}=8 / 255$} \\
			\hline \multirow{2}{*}{ Method } &\multirow{2}{*}{ SA (\%) } &  \multicolumn{2}{c|}{ PGD-10 (\%) }& \multicolumn{2}{c|}{ PGD-50 (\%) }  &\multicolumn{2}{c|}{ AutoAttack (\%) } & Time \\
			&&  \multicolumn{1}{c|}{$\boldsymbol{\epsilon}=8 / 255$} &  \multicolumn{1}{c|}{$\boldsymbol{\epsilon}=16 / 255$}&   \multicolumn{1}{c|}{$\boldsymbol{\epsilon}=8 / 255$ }&  $\boldsymbol{\epsilon}=16 / 255$ &\multicolumn{1}{c|}{$\boldsymbol{\epsilon}=8 / 255$} &  \multicolumn{1}{c|}{$\boldsymbol{\epsilon}=16 / 255$}&  (Sec/ Iteration)\\
			\cmidrule{1-9} 
			
			Fast-AT & $\mathbf{83.56} $\footnotesize{$\pm0.06$ } & $47.03$\footnotesize{$\pm0.29$}& $13.79$\footnotesize{$\pm 0.15$}& $44.94$\footnotesize{$\pm0.52$} &$8.85$\footnotesize{$\pm0.20$} & $ 41.80$\footnotesize{$\pm0.68$ }& $7.32$\footnotesize{$\pm0.27 $}& $5.543\times10^{-2}$ \\
			
			\rowcolor{gray!15} LF-AT (Ours) & $81.70$\footnotesize{$\pm0.15$} &  $\mathbf{47.17}$\footnotesize{$ \pm0.15$} & $\mathbf{14.48}$\footnotesize{$ \pm0.23$}&$\mathbf{45.50} $\footnotesize{$\pm0.04  $}& $\mathbf{9.89} $\footnotesize{$\pm0.14$} & $\mathbf{42.11} $\footnotesize{$\pm0.19$} & $\mathbf{8.13} $\footnotesize{$\pm0.20$}&$5.719\times10^{-2}$ \\
			\cmidrule{1-9}
			
			Fast-AT-GA & $\mathbf{81.00}$\footnotesize{$\pm0.59$} & $48.30$\footnotesize{$\pm0.13 $} &$16.36$\footnotesize{$\pm0.14$}& $ 46.63$\footnotesize{$\pm0.33 $} &$11.12$\footnotesize{$\pm0.12$} & $43.17 $\footnotesize{$\pm0.21$ }& $9.04$\footnotesize{$\pm0.18$} & $1.632\times10^{-1}$ \\
			
			\rowcolor{gray!15} LF-AT-GA (Ours) & $79.18$\footnotesize{$\pm0.13 $} & $\mathbf{48.60} $\footnotesize{$\pm0.06 $}& $\mathbf{17.52} $\footnotesize{$\pm0.02$}& $\mathbf{47.25} $\footnotesize{$\pm0.09 $} &$\mathbf{12.63} $\footnotesize{$\pm0.17$} & $\mathbf{43.31} $\footnotesize{$\pm0.23$} & $\mathbf{10.22} $\footnotesize{$\pm0.05$} &$1.643\times10^{-1}$ \\
			\cmidrule{1-9}
			
			Fast-BAT & $\mathbf{82.01}$\footnotesize{$\pm0.04 $} & $50.42$\footnotesize{$\pm0.36 $ }&$18.29$\footnotesize{$\pm 0.18$}& $ 49.07$\footnotesize{$ \pm0.39$} &$13.31$\footnotesize{$\pm0.16$} & $45.51$\footnotesize{$\pm 0.44 $ }& $10.98$\footnotesize{$\pm 0.19$}&$1.644\times10^{-1}$ \\
			
			\rowcolor{gray!15} LF-BAT (Ours) & $79.72$\footnotesize{$\pm 0.14$} & $\mathbf{50.65}$\footnotesize{$\pm0.19$ } & $\mathbf{19.73}$\footnotesize{$\pm0.05 $ } & $\mathbf{49.66}$\footnotesize{$\pm0.20$} & $\mathbf{15.25}$\footnotesize{$\pm0.20$} & $\mathbf{45.54} $\footnotesize{$\pm0.27$}& $\mathbf{12.23}$\footnotesize{$\pm0.27$} & $1.656\times10^{-1}$ \\
			\hline \bottomrule
			
		\end{tabular}
	}	
\end{table*}

\begin{table}[htbp]
	\centering
	\caption{Illustrating the SA and RA results of Fast-AT, Fast-AT-GA, Fast-BAT and improved version under LAST (denoted as LF-AT, LF-AT-GA and LF-BAT) under PGD attack and AutoAttack on CIFAR100 dataset. We annotate the performance gain when defending against different attacks with subscripts.}
	\label{tab:AT_results2}
	%	\vspace{0.4cm}
	\renewcommand\arraystretch{1.2}
	\setlength{\tabcolsep}{1.2mm}{
		\begin{tabular}{c|c|c|c}
			\toprule \hline 
			\multicolumn{4}{c}{ CIFAR-100 dataset, PARN-18 backbone trained with $\boldsymbol{\epsilon}=8 / 255$} \\
			\hline \multirow{2}{*}{ Method } &\multirow{2}{*}{ SA (\%) }&  \multicolumn{2}{c}{ PGD-10 (\%) }\\
			& & \multicolumn{1}{c|}{$\boldsymbol{\epsilon}=8 / 255$} &  \multicolumn{1}{c}{$\boldsymbol{\epsilon}=16 / 255$} \\
			\cmidrule{1-4} 
			Fast-AT  & $\mathbf{55.08} $& $24.33$& $7.43$  \\
			\rowcolor{gray!15} LFast-AT (Ours)&  $50.81 $ & $\mathbf{25.19}_{\uparrow0.86} $&$\mathbf{9.00}_{\uparrow1.57}  $\\
			\cmidrule{1-4} 
			Fast-AT-GA & $\mathbf{53.25} $ &$25.66$& $ 8.60$ \\
			\rowcolor{gray!15}LFast-AT-GA  (Ours) & $48.22$& $\mathbf{25.88}_{\uparrow0.23}$& $\mathbf{10.27}_{\uparrow1.67}  $ \\
			\cmidrule{1-4} 
			Fast-BAT & $\mathbf{42.79}$ &$22.60$& $8.92$ \\
			\rowcolor{gray!15}LFast-BAT  (Ours)&$42.46$ &$\mathbf{23.15}_{\uparrow0.55}$ & $\mathbf{9.84}_{\uparrow0.92}$ \\
			\hline \bottomrule
			\hline \multirow{2}{*}{ Method }& \multicolumn{2}{c|}{ PGD-50 (\%) }  &\multicolumn{1}{c}{ AutoAttack (\%) } \\
			&   \multicolumn{1}{c|}{$\boldsymbol{\epsilon}=8 / 255$} &  \multicolumn{1}{c|}{$\boldsymbol{\epsilon}=16 / 255$}&   \multicolumn{1}{c}{$\boldsymbol{\epsilon}=16/ 255$ } \\
			\cmidrule{1-4} 
			Fast-AT  &$23.53$& $5.52$ &$4.15$ \\
			\rowcolor{gray!15} LFast-AT (Ours)& $\mathbf{24.37}_{\uparrow0.84} $ & $\mathbf{7.49}_{\uparrow1.98} $&$\mathbf{5.44}_{\uparrow1.29}$\\
			\cmidrule{1-4} 
			Fast-AT-GA &$24.85$ & $6.83$&$5.32$\\
			\rowcolor{gray!15}LFast-AT-GA  (Ours) &$\mathbf{25.43}_{\uparrow0.58}$& $\mathbf{8.81}_{\uparrow1.97}$&$\mathbf{6.27}_{\uparrow0.95}$\\
			\cmidrule{1-4} 
			Fast-BAT &$22.05$ & $7.81$& $5.80$\\
			\rowcolor{gray!15}LFast-BAT  (Ours)& $\mathbf{22.78}_{\uparrow0.72}$& $\mathbf{8.74}_{\uparrow0.93}$& $\mathbf{6.10}_{\uparrow0.30}$ \\
			\hline \bottomrule
		\end{tabular}
	}
\end{table}

\subsubsection{Datasets and models} 

In this paper, we conduct our experiments based on CIFAR10 dataset~\cite{krizhevsky2009learning}, CIFAR100 dataset~\cite{krizhevsky2009learning} and TinyImageNet~\cite{deng2009imagenet} datasets. CIFAR10 and CIFAR100 datasets are foundational benchmarks widely utilized for evaluating robustness of AT trained models. CIFAR-10 consists of 60,000 32x32 color images in 10 classes, while CIFAR-100 maintains the same image dimensions and total number of images as CIFAR-10 and offers a more granular classification challenge with 100 classes. The TinyImageNet dataset extends the challenge by offering 100,000 images across 200 classes, thereby broadening the scope for deeper and more nuanced evaluation. As for the network structures, we mainly use the PreActResNet (PARN)-18~\cite{he2016identity} as our backbone, which is commonly implemented by existing AT methods. Furthermore, we also extend to the larger WideResNet (WRN)-34-10~\cite{zagoruyko2016wide} backbone to demonstrate the generalization performance of LAST framework with large-scale network structure. As for the adversarial attacks to evaluate the robustness, we adopt PGD and AutoAttack~\cite{croce2020reliable} with different attack iterations, perturbation sizes. Specifically, we use PGD-10 and PGD-50 to represent 10-step PGD attack with 1 restart step and 50-step PGD attack with 10 restart steps, respectively.  

\subsubsection{Training Baselines and Metrics} 

We demonstrate LAST's potential to consistently improve various popular SAT methods. For the single-step AT methods, we choose Fast-AT~\cite{wong2020fast}, Fast-AT-GA~\cite{andriushchenko2020understanding} and Fast-BAT~\cite{zhang2022revisiting} as representative SAT methods. Fast-AT-GA refines the Fast-AT through strategies that ensure better gradient quality named GradAligh (GA), while Fast-BAT accesses the second-order gradient based on BLO formulation, which represents the SOTA method for comparison. As for the multi-step AT methods, PGD based AT have been widely explored and used for improving robustness with higher computational cost. Furthermore, we also implement LAST framework on TinyImageNet dataset and compared its performance with PGD-AT, TRADES~\cite{zhang2019theoretically} and DyART~\cite{xu2023exploring}, which is the selected SOTA method from RobustBench~\cite{croce2021robustbench}.  Note that these methods for comparison are also implemented with SWA techniques to show the effectiveness of proxy guided update rule. As for the evaluation metrics, we adopt the test Robust Accuracy (RA) as the metric of robustness evaluation. It's important to note that, unless otherwise specified, the LAST framework is implemented without the SD objective when reporting performance. This ensures that the comparisons are made under the same defensive objectives and common hyperparameters, allowing for a fair comparison of the methods' effectiveness in adversarial robustness. In particular, we also analyze the convergence behavior of test robust loss and RA to evaluate the robustness performance and stability of training process. Besides, we report the average performance of the best model trained with early stopping under three random seeds for fair comparison. As for the  calculation of average runtime, we report the runtime each iteration within the firs epoch to compare the computation efficiency of original AT methods and our improved versions. 

\subsubsection{Training and Hyperpramerters} 

Basically, we follow~\cite{zhang2022revisiting,xu2023exploring} to set most common hyperparameters such as the $\boldsymbol{\alpha}$ and $\boldsymbol{\beta}$ for different adversarial attacks. All the methods adopt the SGD optimizer with momentum set as 0.9 and weight decay set as $5.e^{-4}$. For PARN-18 backbone, we train Fast-AT, Fast-AT-GA, Fast-BAT and PGD-AT(-2) (i.e., 2-step PGD-AT) for 30 epochs and use cyclic scheduler with maximum $\boldsymbol{\beta}=0.2$, and train PGD-AT(-10) for 200 epochs using multi-step scheduler with $\boldsymbol{\beta}=0.1$. For WRN-34-10 backbone, we train Fast-AT, Fast-AT-GA and Fast-BAT for 50 epochs using the same learning rate schedular with $\boldsymbol{\beta}=0.1$. For the special hyperparameter of LAST framework, we set the aggregation coefficient $\boldsymbol{\gamma}=0.8$.  Note that the introduced SD objective is only applied for Fast-AT to verify its effectiveness to alleviate the catastrophic overfitting problem under larger perturbation $\boldsymbol{\epsilon}=16/255$. As for the distillation coefficient and temperature for the self distillation defense objective, we follow ~\cite{liexploring} to set the common distillation hyperparameter for PARN as $\boldsymbol{\mu}=0.95$ and $\boldsymbol{\tau}=6.0$. When compared with SOTA methods, we replicate TRADES and DyART according to the implementation of DyART, which first trains the model with clean data with 42 epochs, and then train the model with 100 epochs by loading the checkpoint at 20 epoch. 

%~\footnote{\url{https://github.com/haitongli/knowledge-distillation-pytorch}.}~\footnote{\url{https://github.com/OPTML-Group/Fast-BAT}.}

\subsection{Evaluation with Single-Step AT methods}
We first demonstrate the improvement of LAST framework on the training stability and robustness with various single-step AT methods including Fast-AT, Fast-AT-GA and Fast-BAT.

 \begin{figure*}[!t] 
	\begin{center}
		\begin{tabular}{@{\extracolsep{-2.3em}}c@{\extracolsep{-0.2em}}c@{\extracolsep{-0.0em}}c@{\extracolsep{-0.3em}}c@{\extracolsep{-0.2em}}}
			\multicolumn{2}{c}{\includegraphics[height=3.97cm,width=7.1cm,trim=20 0 35 0,clip]{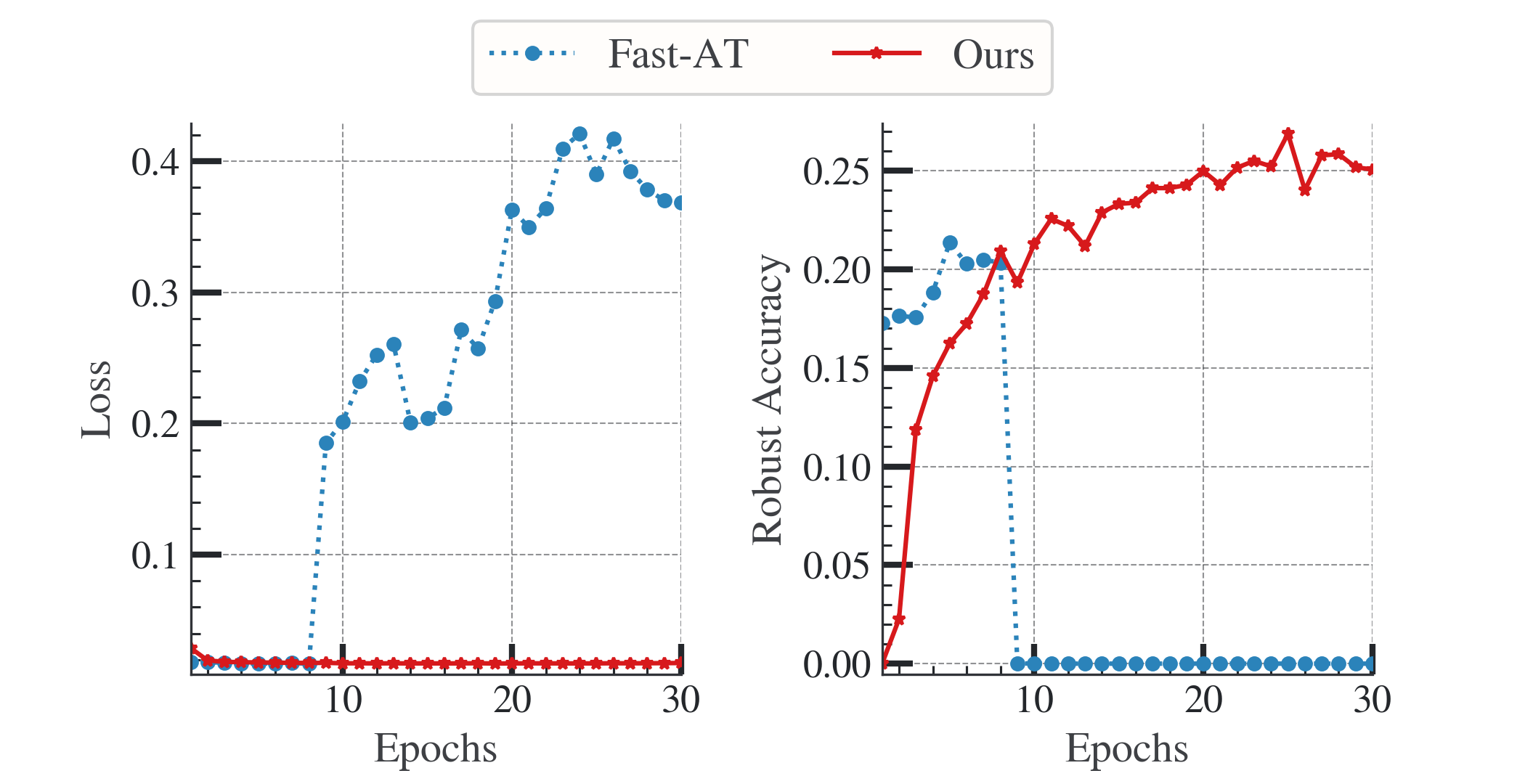} }&\multicolumn{2}{c}{\includegraphics[height=3.67cm,width=9.8cm,trim=20 0 0 30,clip]{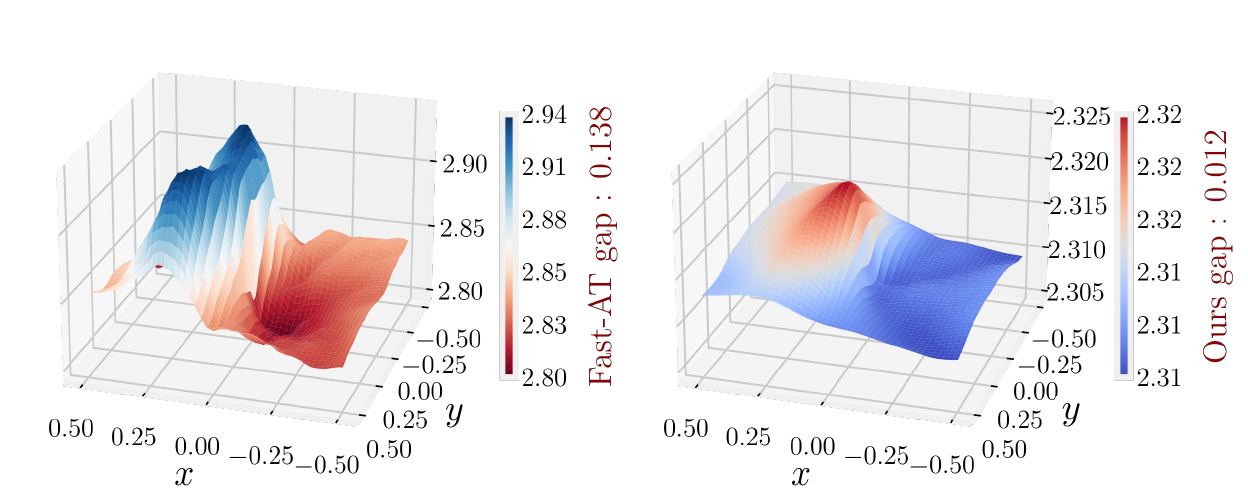} }\\
			\multicolumn{2}{c}{ (a) Robust Loss and Accuracy}& \multicolumn{2}{c}{ (b) Adversarial Loss Landscape} \\
		\end{tabular}
	\end{center}
	\caption{Subfigure (a) compare the convergence behavior of test loss and RA for Fast-AT and ours on CIFAR10 dataset under PGD-10 attack with $\boldsymbol{\epsilon}=16/255$. In Subfigure (b), we compare the loss landscape for Fast-AT and our version. The gap of maxi- and minimum losses is calculated within the range of $x, y\in[-0.5,0.5]$.}\label{fig:fast_at}
\end{figure*}

\textbf{Robust Accuracy on the CIFAR10 dataset.} We first implement the LAST framework without SD defense objective based on Fast-AT, Fast-AT-GA, Fast-BAT and evaluate the Standard Accuracy (SA) and RA on the CIFAR10 using PARN-18 backbone in Tab.~\ref{tab:AT_results1}.  Note that all the results are calculated by running the training process with three random seeds, with the evidence that the performance of Fast-BAT will make a significant difference using its default seed and other random seeds during training. We use $m${\footnotesize$\pm n$} to represent the mean SA (i.e., $m$) with standard deviation (i.e., {\footnotesize$n$}). As it is shown, when evaluating the robustness using the same attack strength($\boldsymbol{\epsilon}=8 / 255$), our improved version of three SAT methods consistently improves the Robust Accuracy (RA) of Fast-AT, FAST-AT-GA and Fast-BAT. More importantly, the RA of our LAST framework has a significantly lower standard deviation,  indicating that the random seed has less impact on the training performance. 

Furthermore, when tested against stronger attacks ($\boldsymbol{\epsilon}=16 / 255$), our method shows significantly better improvements in RA (and in most cases lower standard deviation), which indicates that the LAST framework effectively improves the model's robustness to adversarial examples introduced with unknown adversarial perturbations, rather than just being robust to the magnitude of the trained adversarial perturbations. This conclusion can also be seen by the analysis of the adversarial loss landscape in Fig.~\ref{fig:loss_landscape}, where our improved method has significantly lower loss on average as well as smoother surfaces.

\textbf{Runtime analysis.} Furthermore, we evaluate the average runtime of SAT methods and our improve ones for each iteration to analyze the extra computation cost brought by the proxy guided update rule. In Algorithm 1, we can see that the extra calculation caused by the introduced proxy model is mainly the calculation of the differential unit $\mathcal{G}_{\boldsymbol{\theta}}$, and the gradient calculation burden about the adversarial example is similar to that of SAT. As it is shown in the table, improving existing AT methods with our framework only slightly increases the runtime (up to $2\times10^{-3}$ Sec/ Iteration, 3.17\% increase), which demonstrates its potential to serve as an alternative of SAT with compromised computation cost.

 \begin{table}[!t]
 	\centering
 	\caption{We report the SA and RA of Fast-AT, Fast-AT-GA, Fast-BAT and the LAST framework under PGD attack (PGD-10 and PGD-50) and AutoAttack on CIFAR10 dataset using WRN-34-10 structure as the backbone. }
 	\label{tab:AT_results6}
 	\renewcommand\arraystretch{1.4}
 	\setlength{\tabcolsep}{0.4mm}{
 		\begin{tabular}{c|c|c|c|c|c|c}
 			\toprule \hline \multicolumn{7}{c}{ CIFAR-10 dataset, WRN-34-10 backbone trained with $\boldsymbol{\epsilon}=8 / 255$}\\
 			\hline \multirow{2}{*}{ Method } &\multirow{2}{*}{ SA (\%) }&  \multicolumn{2}{c|}{ PGD-10 (\%) }& \multicolumn{2}{c|}{ PGD-50 (\%) }  &\multicolumn{1}{c}{ AutoAttack (\%) } \\
 			& & \multicolumn{1}{c|}{$8 / 255$} &  \multicolumn{1}{c|}{$16 / 255$}&  \multicolumn{1}{c|}{$8 / 255$} &  \multicolumn{1}{c|}{$16 / 255$}&     \multicolumn{1}{c}{$16/ 255$ } \\	
 			\cmidrule{1-7} 
 			Fast-AT  & $\mathbf{80.00}$& $45.89$& $17.49$ &$43.65$& $10.92$&$7.80$\\
 			%			\rowcolor{gray!15} LF-AT &  $50.817 $\footnotesize{$\pm0.33 $} & $\mathbf{25.190}$\footnotesize{$\pm0.07 $}&$\mathbf{9.003}$\footnotesize{$\pm0.07 $}& $\mathbf{24.373}$\footnotesize{$\pm0.42 $} & $\mathbf{7.497}$\footnotesize{$\pm0.09 $}&$\mathbf{5.443}$\footnotesize{$\pm0.13$}\\
 			%			\cmidrule{1-7}
 			Fast-AT-GA & $78.72$ &$46.82$& $ 18.01$ &$45.12$& $12.31$&$9.82$\\
 			%			\rowcolor{gray!15}LF-AT-GA  & $48.220 $\footnotesize{$\pm0.32 $}& $\mathbf{25.887}$\footnotesize{$\pm0.14 $}& $\mathbf{10.277}$\footnotesize{$\pm0.10 $} &$\mathbf{25.433}$\footnotesize{$\pm0.23 $}& $\mathbf{8.813}$\footnotesize{$\pm0.14 $}&$\mathbf{6.270}$\footnotesize{$\pm0.12 $}\\
 			%			\cmidrule{1-7}
 			Fast-BAT & $79.93$ &$47.87$& $17.55$ &$46.45$& $12.41$& $10.09$\\
 			\cmidrule{1-7}
 			\rowcolor{gray!15}LAST (Ours) &$77.88$&$\mathbf{49.02}$ & $\mathbf{19.23}$&$\mathbf{47.94}$& $\mathbf{14.15}$&$\mathbf{11.87}$ \\
 			\hline \bottomrule
 		\end{tabular}
 	}
 \end{table}
 
\textbf{Evaluation on the CIFAR100 dataset.} In Tab.~\ref{tab:AT_results2}, we compare the performance of these methods with our LAST framework on CIFAR100 dataset. Note that we annotate the performance gain for different methods when defending against different attacks with subscripts. It can be clearly observed that the LAST framework shows consistent performance improvement of RA on PGD-10, PGD-50 and AutoAttack. Besides, we can find that our improved versions of Fast-AT, Fast-AT-GA and Fast-BAT have more significant performance gains of average performance on the CIFAR 100 dataset with more categories compared with CIFAR10 dataset.

\textbf{Evaluation on the WRN-34-10 backbone.} Besides, we also demonstrate the effectiveness of LAST framework with different backbones by implementing these methods and our LAST framework based on the larger WRN-34-10 backbone on CIFAR10 dataset. As it is shown  in Tab.~\ref{tab:AT_results6}, the LAST framework also achieves significant robustness improvement compared to Fas-AT, Fast-AT-GA and Fast-BAT against different attack types and perturbation sizes. In addition, we also noticed that the significant improvement in model robustness is accompanied by a decrease in clean accuracy. This trade-off has been widely discussed as another common issue in SAT, and many techniques have been investigated specifically aiming at improving clean accuracy without performance loss such as using the CAM technique~\cite{jiang2021layercam} or class-agnostic methods~\cite{ribeiro2016should}, which are not only applicable to SAT but also can be extended to our proposed LAST framework.
	\begin{table*}[!t]
	\centering
	\caption{We report the SA and RA of PGD-10 and our improved version (LPGD-10) under PGD attack (PGD-10 and PGD-50) and AutoAttack on CIFAR10 and CIFAR100 dataset. We use $m${\footnotesize$\pm n$} to denote the mean SA (i.e., $m$) with standard deviation (i.e., {\footnotesize$n$}).}
	
	\label{tab:AT_results3}
	\renewcommand\arraystretch{1.4}
	\setlength{\tabcolsep}{1.2mm}{
		\begin{tabular}{c|c|c|c|c|c|c|c}
			\toprule \hline \multicolumn{8}{c}{ CIFAR-10 dataset, PARN-18 backbone trained with $\boldsymbol{\epsilon}=8 / 255$} \\
			\hline \multirow{2}{*}{ Method } &\multirow{2}{*}{ SA (\%) }&  \multicolumn{2}{c|}{ PGD-10 (\%) }& \multicolumn{2}{c|}{ PGD-50 (\%) }  &\multicolumn{2}{c}{ AutoAttack (\%) } \\
			& & \multicolumn{1}{c|}{$\boldsymbol{\epsilon}=8 / 255$} &  \multicolumn{1}{c|}{$\boldsymbol{\epsilon}=16 / 255$}&  \multicolumn{1}{c|}{$\boldsymbol{\epsilon}=8 / 255$} &  \multicolumn{1}{c|}{$\boldsymbol{\epsilon}=16 / 255$}&  \multicolumn{1}{c|}{$\boldsymbol{\epsilon}=8 / 255$} &  \multicolumn{1}{c}{$\boldsymbol{\epsilon}=16/ 255$ } \\
			\cmidrule{1-8} 
			PGD-AT & $81.948 $\footnotesize{$\pm0.74$}& $51.923$\footnotesize{$\pm0.30 $}& $20.310$\footnotesize{$\pm0.75 $}&$50.757$\footnotesize{$\pm0.33$}& $15.677$\footnotesize{$\pm0.41 $} &$47.087$\footnotesize{$\pm0.54 $}&$13.093$\footnotesize{$\pm0.43 $} \\
			\rowcolor{gray!15} LPGD-AT (Ours) &  $\mathbf{82.187}$\footnotesize{$\pm0.90$} & $\mathbf{53.230}$\footnotesize{$\pm0.20$}&$\mathbf{22.203}$\footnotesize{$\pm0.37$}& $\mathbf{52.137}$\footnotesize{$\pm0.10$} & $\mathbf{17.587}$\footnotesize{$\pm0.57$} &$\mathbf{47.707}$\footnotesize{$\pm0.22 $}&$\mathbf{14.297}$\footnotesize{$\pm0.03$}\\
			
			\hline \multicolumn{8}{c}{ CIFAR-100 dataset, PARN-18 backbone trained with $\boldsymbol{\epsilon}=8 / 255$} \\
			\hline \multirow{2}{*}{ Method } &\multirow{2}{*}{ SA (\%) }&  \multicolumn{2}{c|}{ PGD-10 (\%) }& \multicolumn{2}{c|}{ PGD-50 (\%) }  &\multicolumn{2}{c}{ AutoAttack (\%) } \\
			& & \multicolumn{1}{c|}{$\boldsymbol{\epsilon}=8 / 255$} &  \multicolumn{1}{c|}{$\boldsymbol{\epsilon}=16 / 255$}&  \multicolumn{1}{c|}{$\boldsymbol{\epsilon}=8 / 255$} &  \multicolumn{1}{c|}{$\boldsymbol{\epsilon}=16 / 255$}&    \multicolumn{1}{c|}{$\boldsymbol{\epsilon}=8 / 255$} &\multicolumn{1}{c}{$\boldsymbol{\epsilon}=16/ 255$ } \\
			\cmidrule{1-8} 
			PGD-AT & $\mathbf{49.457}$\footnotesize{$\pm0.48$}& $25.837$\footnotesize{$\pm0.43$}& $9.980$\footnotesize{$\pm0.43$}&$25.377$\footnotesize{$\pm0.39$}& $8.749$\footnotesize{$\pm0.38$} &$21.109$\footnotesize{$\pm0.05$}
			&$6.667$\footnotesize{$\pm0.23$} \\
			\rowcolor{gray!15} LPGD-AT (Ours) &  $48.150$\footnotesize{$\pm0.42$} & $\mathbf{31.267}$\footnotesize{$\pm0.53$}&$\mathbf{14.903}$\footnotesize{$\pm0.26$}& $\mathbf{30.857}$\footnotesize{$\pm0.62$} & $\mathbf{13.573}$\footnotesize{$\pm0.52$}& $\mathbf{23.113}$\footnotesize{$\pm0.18$}&$\mathbf{8.033}$\footnotesize{$\pm0.40$}\\
			
			\hline \bottomrule
			
		\end{tabular}
	}
\end{table*}

\begin{figure*}[htbp]
	%			\vspace{-0.3cm}
	\begin{center}
		%		\renewcommand\arraystretch{0.5}
		%		\begin{tabular}{@{\extracolsep{-0.3em}}c@{\extracolsep{-0.1em}}}
			\includegraphics[height=4.0cm,width=18cm,trim=80 0 80 0,clip]{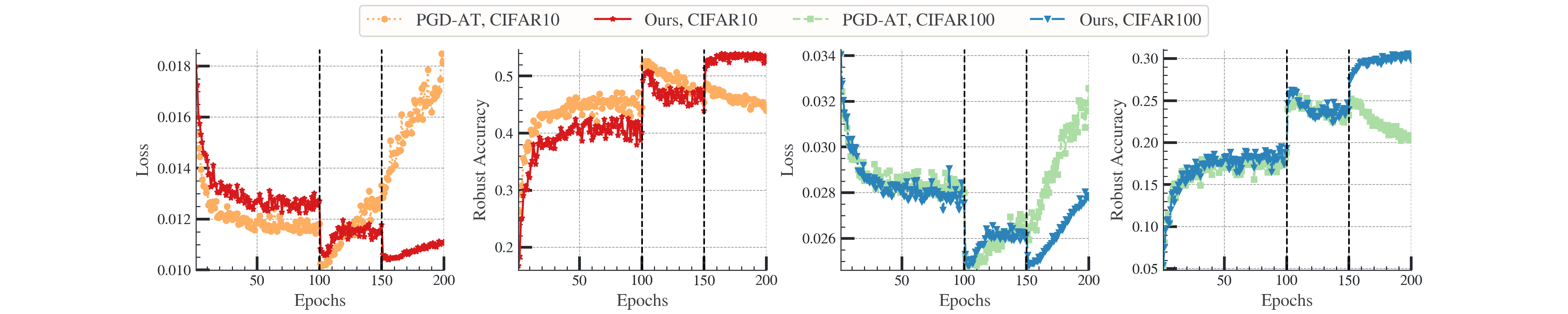}\\
			%\specialrule{0em}{1pt}{3pt}
			%			\specialrule{0em}{0pt}{0pt}
			%			\footnotesize \quad\quad\quad\quad\quad(a) CIFAR10 \quad\quad &\footnotesize \quad\quad\quad(b) CIFAR100 \quad\quad \\
			%		\end{tabular}
	\end{center}
	\caption{The four subfigures compare the convergence behavior of test robust loss and RA trained with PGD-AT and LAST, $\boldsymbol{\epsilon}=8/255$ on CIFAR10 dataset and CIFAR100 dataset. The black dashed line denotes the epoch where multi-step learning rate decays.}\label{fig:PGD-10}
\end{figure*} 

\textbf{Alleviating catastrophic overfitting with SD.} Furthermore, in Fig.~\ref{fig:fast_at}, we investigate the phenomenon of catastrophic overfitting when facing stronger adversaries. It is evident that during the AT process with large perturbation sizes, the robustness loss sharply increases, and the robust accuracy (RA) suddenly drops to 0. Additionally, we visualize the adversarial loss landscape of the best model obtained using Fast-AT with early stopping. The loss gaps for Fast-AT and ours within the range of $x, y\in[-0.5,0.5]$ are also reported along the color bar. Notably, influenced by the injected perturbation and random noise, the loss surface exhibits high values and lacks smoothness, resulting in a significant loss gap (More detailed definition about the loss landscape could refer to Fig.~\ref{fig:loss_landscape}. As a point of comparison, we also report on the convergence behavior and loss landscape after incorporating the SD loss. Within the LAST framework, the surrogate model’s update direction is adjusted based on the prior information distilled from soft labels. This correction leads to a more stable convergence behavior, an enhancement in robustness, and a smoother loss surface overall.

\subsection{Evaluation with Multi-Step AT methods}
 In this subsection, we demonstrate the effectiveness of LAST by improving multi-step AT methods across different datasets and attacks.

\textbf{Quantitative results on commonly used datasets.} In this subsection, we improve the PGD-based adversarial training method under the proposed LAST framework, thereby substantiating the efficacy of LAST in the realm of multi-step AT techniques. In Tab.~\ref{tab:AT_results3}, we first report the SA and RA of PGD-AT and our improved version (abbreviated as LPGD-AT) on the CIFAR10 and CIFAR100 datasets. We also use $m${\footnotesize$\pm n$} to report the average accuracy and standard deviation of different methods. From the table, we can see that although the improved version under the LAST framework still has a reduced clean accuracy, it has achieved more significant performance improvements on PGD and AutoAttack with different perturbation sizes (in the last column of the table, it achieved a $9.2\%$ and $20.3\%$ accuracy improvement relative to PGD-AT on the two datasets), and still shows lower standard deviation in most results. In a particularly compelling demonstration of LAST's potential, the CIFAR10 dataset not only witnesses an elevation in RA but also enjoys an enhancement in clean accuracy, thereby underscoring the LAST framework's capacity to significantly bolster a wide array of multi-step AT methods.

\textbf{Training convergence behavior.} In addition, we analyze the convergence behavior of robust loss and RA for both PGD-AT and LPGD-AT on the CIFAR10 and CIFAR100 datasets in Fig.~\ref{fig:PGD-10}. It is worth noting that we adopt 200 epochs widely used by existing methods for training multi-step AT methods. The two black vertical dashed lines in the figure denote the epochs at which multi-step learning rate decay is applied. It is evident that for both datasets, PGD-AT achieves relatively optimal robust loss and robust accuracy after the first learning rate decay, with subsequent training failing to enhance performance. In contrast, the LPGD-AT, improved by the LAST framework, enhances the consistency between the model's historical parameter states, resulting in a smoother curve and more stable convergence in robust loss and RA. Meanwhile, LPGD-AT achieves higher performance after two learning rate decays. Furthermore, as depicted in Fig.~\ref{fig:loss_landscape}, LPGD-AT exhibits a smoother surface and smaller loss gaps, further validating the LAST framework's ability to enhance the model's robustness against varying adversarial perturbations and random perturbations in adversarial samples.

%In addition, we compare the convergence behavior of robust loss and RA for PGD-AT and our LPGD-AT on both CIFAR10 and CIFAR100 datasets in Fig.~\ref{fig:PGD-10}. As it is illustrated, by improving the consistency among the historical states of model parameters, LPGD-AT exhibits more stable convergence behavior of both robust loss and accuracy, and finally gains higher performance after performing the multi-step learning rate decay twice. Besides, as it can be observed in subfigure (d) of  Fig.~\ref{fig:loss_landscape}, the loss landscape of model trained with LAST framework has been rendered smoother, accompanied by a reduced loss gap between its highest and lowest values, which demonstrates that the model obtained by LAST is more robust in the presence of different strengths of adversarial perturbations and random perturbations injected in the adversarial samples.

\begin{table}[h!]
	\centering
	\caption{We report the SA and RA of PGD-AT, TRADES, DyART and LAST (Ours) under different perturbation sizes of AutoAttack on TinyImageNet dataset based on PARN-18 backbone. }
	\label{tab:AT_results4}
	\renewcommand\arraystretch{1.4}
	\setlength{\tabcolsep}{0.8mm}{
		\begin{tabular}{c|c|c|c|c}
			\toprule \hline \multicolumn{5}{c}{ TinyImageNet dataset, PARN-18 backbone trained with $\boldsymbol{\epsilon}=8 / 255$} \\
			\hline \multirow{2}{*}{ Method } &\multirow{2}{*}{ SA (\%) }&  \multicolumn{3}{c}{ RA (\%), AutoAttack } \\
			&& \multicolumn{1}{c|}{$\boldsymbol{\epsilon}=2 / 255$} &  \multicolumn{1}{c|}{$\boldsymbol{\epsilon}=4 / 255$}&  \multicolumn{1}{c}{$\boldsymbol{\epsilon}=8 / 255$}\\
			\cmidrule{1-5}
			PGD-AT &$48.09$&$38.82$&$30.18$&$16.46$ \\
			TRADES & $46.68$&$37.84$&$29.85$&$16.76$\\
			%		MART&$45.51$&$36.68$&$29.15$&$17.79$ \\
			DyART&$48.38$&$38.46$&$29.69$&$17.52$\\
			\cmidrule{1-5} 
			\rowcolor{gray!15} LAST (Ours)&  $\mathbf{49.07} $ & $\mathbf{39.38}$&$\mathbf{30.54} $& $\mathbf{17.57}$ \\
			%			\rowcolor{gray!15} LAST (Ours)&  $\mathbf{50.85} $ & $\mathbf{40.64}$&$\mathbf{31.54} $& $\mathbf{17.86}$ & $\mathbf{10.05}$&$\mathbf{5.87}$\\
			%						\rowcolor{gray!15} LAST (Ours)&  $45.74 $ & $\mathbf{39.64}$&$\mathbf{30.54} $& $\mathbf{17.86}$ & $\mathbf{10.05}$&$\mathbf{5.87}$\\
			\hline \bottomrule
		\end{tabular}
	}
\end{table}

\textbf{Quantitative comparison with SOTA methods.} Furthermore, we have extended the implementation of the LAST framework to the the larger TinyImageNet dataset, where we conduct the comparison of its performance in terms of SA and RA against several established AT methods, including PGD-AT, TRADES, and DyART. TRADES stands out as a representative method that builds upon the foundation of PGD-AT, while DyART distinguishes itself by utilizing different types of attack and defense process that is distinct from SAT, and has achieved the best robust accuracy based on the PreResNet-18 backbone in the RobustBench rankings. It is important to highlight that the SWA technique, which we previously discussed as a method that utilizes historical model weights, has been incorporated into these methods. This integration serves to further validate the efficacy of the LAST framework's approach to guiding model updates in the context of adversarial defense. In Tab.~\ref{tab:AT_results4}, our enhanced implementation not only demonstrates superior SA but also consistently outperforms the other methods across all three attack strengths. This consistent performance improvement underscores the LAST framework's potential to significantly improve upon current multi-step AT methodologies. 
%
%Additionally, we have implemented the LAST framework on the larger TinyImageNet dataset and compared its clean accuracy and robust accuracy against PGD-AT, TRADES, and DyART under AutoAttack with varying perturbation strengths. TRADES is a representative method based on PGD-AT, while DyART employs an attack and defense update scheme different from SAT and has achieved the best robust accuracy based on the PreResNet-18 architecture in the RobustBench performance rankings. It is noteworthy that the SWA technique, which leverages historical weights as discussed earlier, is also integrated into the relevant methods, thereby further substantiating the effectiveness of the LAST framework's proxy model-guided update rules in adversarial defense. In Table \ref{tab:AT_results4}, we compare the clean accuracy and robust accuracy of different methods under three levels of perturbation strength. It is readily observable that our improved method not only achieves better clean accuracy but also outperforms others under all three attack strengths, thereby demonstrating the potential of the LAST framework to enhance existing multi-step adversarial training (AT) methods.
%
%此外，我们同样在更大的TinyImageNet 数据集上实现了LAST框架，并与PGD-AT, TRADES 和DyART在的不同扰动强度的AutoAttack下进行了干净精度和鲁棒精度的对比，其中TRADES是基于PGD-AT的代表性方法，而DyART使用与SAT不同的攻击与防御更新方案，并且基于PreResNet-18结构在RobustBench的性能排行中取得了最好的鲁棒精度。值得注意的是，在前文的讨论中，我们提到的利用历史权重的SWA技术同样集成在相关的方法中，从而可以进一步证明LAST框架下代理模型引导的更新规则在对抗防御中的有效性。在表Tab.~\ref{tab:AT_results4}中，我们比较了不同方法的干净精度以及在三种扰动强度下的鲁棒精度。很容易观察到，我们的改进方法不仅获得了更好的干净精度，同时也在三种不同的攻击强度下均取得了更好的性能，从而证明了LAST框架改进现有多步AT方法的潜力。

\subsection{Ablation Results}
In this subsection, we design ablation experiments to verify the effectiveness of the SD defense objective function as well as the aggregation coefficient $\boldsymbol{\gamma}$, respectively.

\begin{figure}[h!]
	\begin{center}
		\begin{tabular}{@{\extracolsep{0.2em}}c@{\extracolsep{0.2em}}c@{\extracolsep{-0.1em}}}
			\multicolumn{2}{c}{\includegraphics[height=0.4cm,width=8.2cm,trim=220 270 200  0,clip]{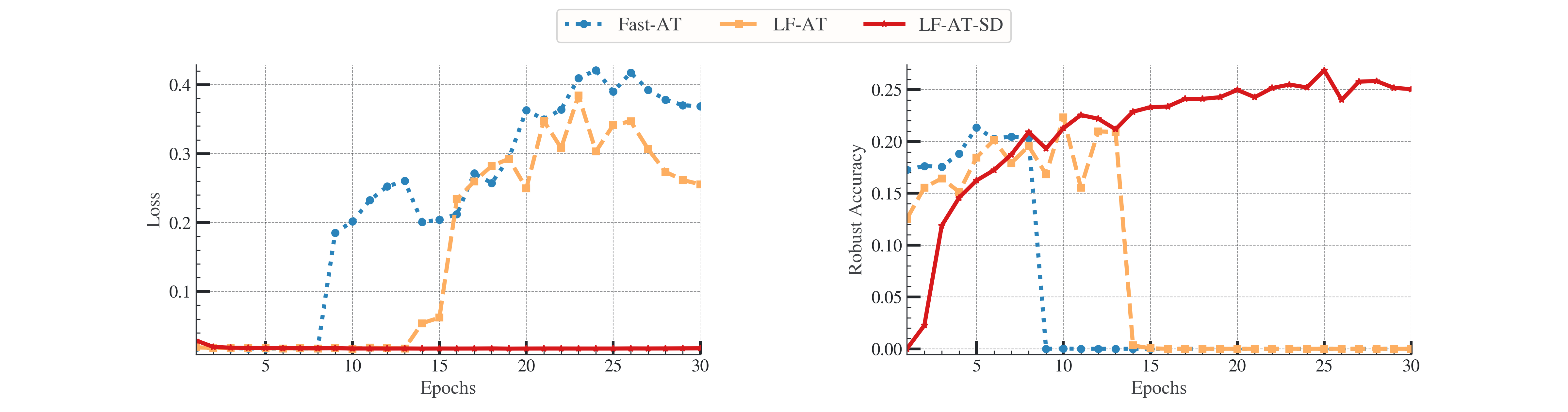} } \\
			\includegraphics[height=3cm,width=4.4cm,trim=0 0 0  0,clip]{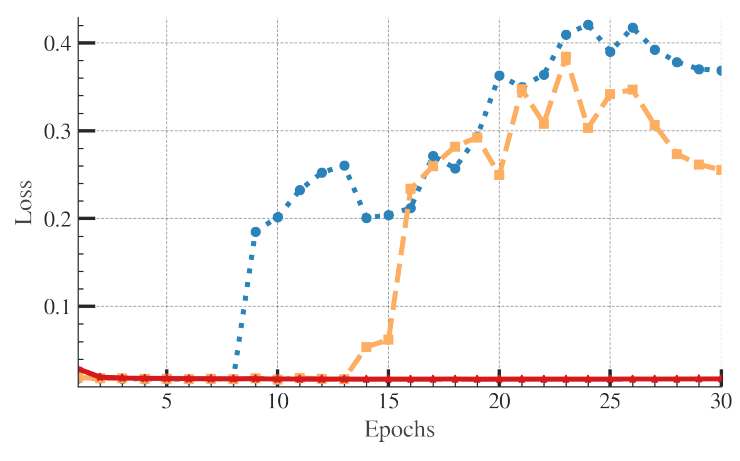}  &
			\includegraphics[height=3cm,width=4.4cm,trim=0 0 0  0,clip]{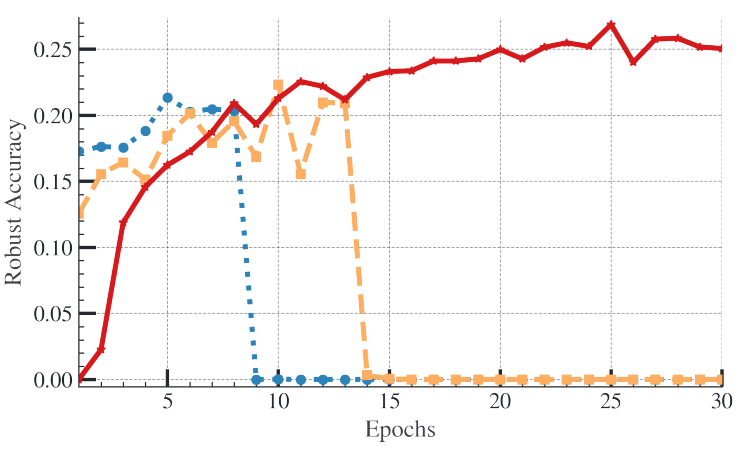}  \\
			%\specialrule{0em}{1pt}{3pt}
			%			\specialrule{0em}{0pt}{0pt}
			%			\footnotesize \quad\quad\quad\quad\quad(a) CIFAR10 \quad\quad &\footnotesize \quad\quad\quad(b) CIFAR100 \quad\quad \\
		\end{tabular}
	\end{center}
	\vspace{-0.3cm}
	\caption{We compare the convergence behavior of Fast-AT, the improved version under LAST framework with and without SD defense objective (denoted as LF-AT and LF-AT-SD, respectively).}\label{fig:ablation_SD}
\end{figure} 

\textbf{Ablation results of the SD defense objective.}  In Figure \ref{fig:ablation_SD}, we present more detailed comparison of the convergence of test robust loss and RA to underscore the impact of incorporating the SD defense objective within the LAST framework on the CIFAR10 dataset, with perturbation size of $\boldsymbol{\epsilon} = 16/255$. The results indicate that the integration of the LAST framework alone has a mitigating effect on the degradation of loss and accuracy during training, evidenced by a slower decline in performance, which suggests that the LAST framework contributes to a more stable training process. However, the most significant improvements are observed when the LAST framework is combined with the SD defense objective. The inclusion of the SD objective leads to a consistent enhancement in the convergence behavior of both test robust loss and RA. These results highlight the potential of the LAST framework to serve as a versatile and powerful tool in the adversarial defense process, particularly when paired with additional techniques like the SD defense objective.

%In Fig.~\ref{fig:ablation_SD}, we provide more detailed comparative results of the convergence behavior of test robust loss and RA to demonstrate the effectiveness of SD objective on CIFAR10 dataset with $\boldsymbol{\epsilon}=16/255$, which serves as a supplement to Fig.~3. As it is shown, combining the LAST framework will slow down the collapse of loss and accuracy and improve the best performance to some extent. When we integrate the LAST framework together with SD defense objective, the convergence behavior of test robust loss and RA are consistently improved and leads to boost of robustness. Unless specified otherwise, we only implement the LAST framework without SD objective to report the performance based on the same configuration of defense objectives and common hyperparameters.

\begin{figure}[h!]
	\begin{center}
		\renewcommand\arraystretch{0.8}
		\begin{tabular}{c@{\extracolsep{0.1em}}}
			\includegraphics[height=3.6cm,width=9.cm,trim=20 0 0  0,clip]{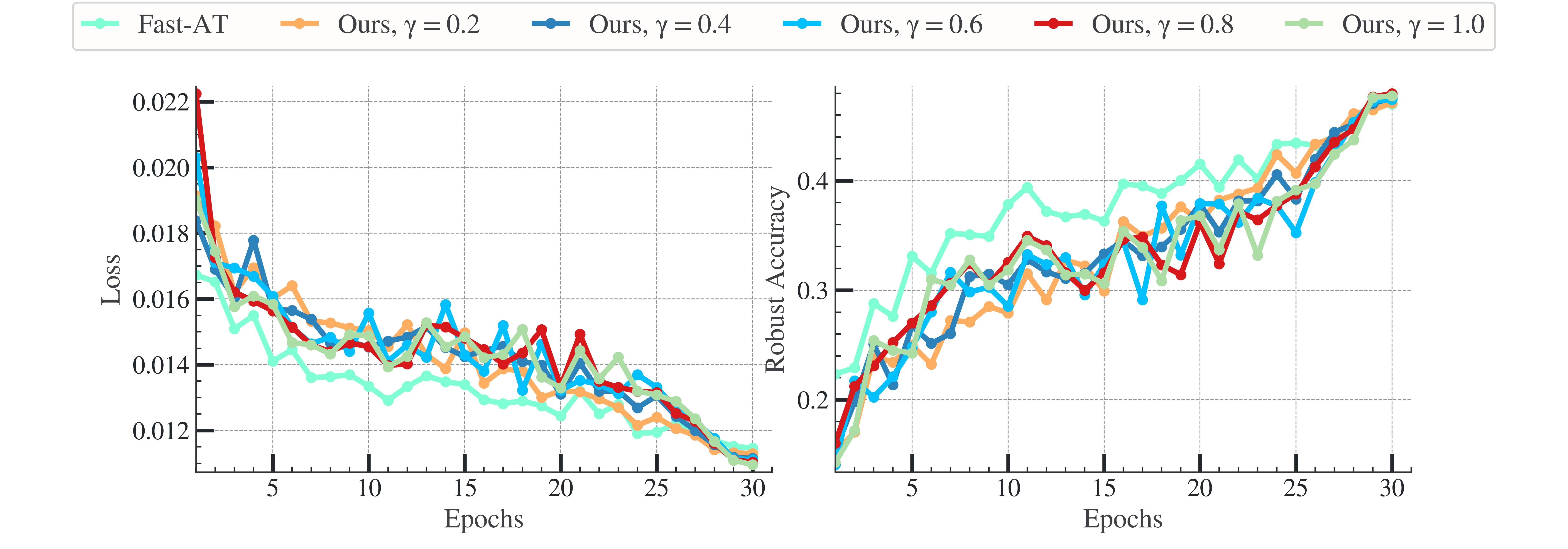} \\
			%\specialrule{0em}{1pt}{3pt}
			\specialrule{0em}{0pt}{0pt}
%			\footnotesize (a) Validation Loss \\
		\end{tabular}
	\end{center}
	%\vspace{-0.4cm}
	\caption{Illustration of convergence behavior of Fast-AT and our improved version under LAST framework as $\boldsymbol{\gamma}$ varies. We selected five discrete values from (0.0, 1.0] for ablation.}\label{fig:ablation}
\end{figure}

\textbf{Ablation results of $\boldsymbol{\gamma}$.} In Fig.~\ref{fig:ablation}, we conduct detailed examination of the impact of the aggregation coefficient $\boldsymbol{\gamma}$ specific to the LAST framework, on the PARN-18 model trained on the CIFAR10 dataset with the perturbation size of $\boldsymbol{\epsilon} = 8/255$. The figure provides insights into how varying $\boldsymbol{\gamma}$ influences the convergence behavior of our improved version in comparison to the Fast-AT baseline. The results reveal several key points:
\begin{itemize}
	\item Initially, the improved method exhibits a slightly slower convergence rate compared to Fast-AT. However, as training proceeds, it manages to achieve lower robust loss  and higher RA. This suggests that while the initial learning pace may be conservative, the LAST framework is capable of assisting the model to better withstand adversarial attacks over time.
	\item Across the range of $\boldsymbol{\gamma}$ values used for ablation, our method consistently outperforms Fast-AT. This indicates that the LAST framework, which incorporates the aggregation coefficient, is effective in enhancing the robustness of the model against adversarial perturbations.
	\item Specifically, there is a slight dip in RA when $\boldsymbol{\gamma}$ transitions from $0.8$ to $1.0$. This suggests that while a higher value of $\boldsymbol{\gamma}$ can lead to better performance, there may be an optimal range for this hyperparameter that balances the speed of convergence with the ultimate robustness achieved.
\end{itemize}
Based on these observations, we have chosen to set $\boldsymbol{\gamma} = 0.8$ for the quantitative experiments presented earlier, unless otherwise specified. This decision reflects a strategic balance between the need for rapid convergence and the desire to maximize robust accuracy. By selecting this value, we aim to optimize the performance of LAST framework, ensuring that it can effectively generalize its robustness to a wide range of adversarial scenarios.

%In Fig.~\ref{fig:ablation}, we further investigate the influence of hyperparameter unique to our method, the aggregation coefficient $\boldsymbol{\gamma}$ based on PARN-18 on CIFAR10 dataset, $\boldsymbol{\epsilon}=8/255$. Firstly, it can be observed that also the improved method converges a bit slower than Fast-AT at the beginning, it gains lower loss and higher RA as the training proceeds. Secondly, although our methods always gain better performance than Fast-AT as $\boldsymbol{\gamma}$ changes, the RA slightly decreases when $\boldsymbol{\gamma}$ approaches $1$ ($0.8\rightarrow1.0$). Based on the above observation, we set $\boldsymbol{\gamma}=0.8$ for the above quantitative experiments unless specified otherwise.

\subsection{Evaluation of Generalization Performance}

In this subsection, we verify the generalization capability of the LAST framework by its performance on black-box attacks as well as on out-of-distribution samples.

\begin{figure}[h!]
	\begin{center}
		\begin{tabular}{@{\extracolsep{0.2em}}c@{\extracolsep{-0.2em}}c@{\extracolsep{-0.1em}}}
			\includegraphics[height=3.75cm,width=4.5cm,trim=0 0 0   0,clip]{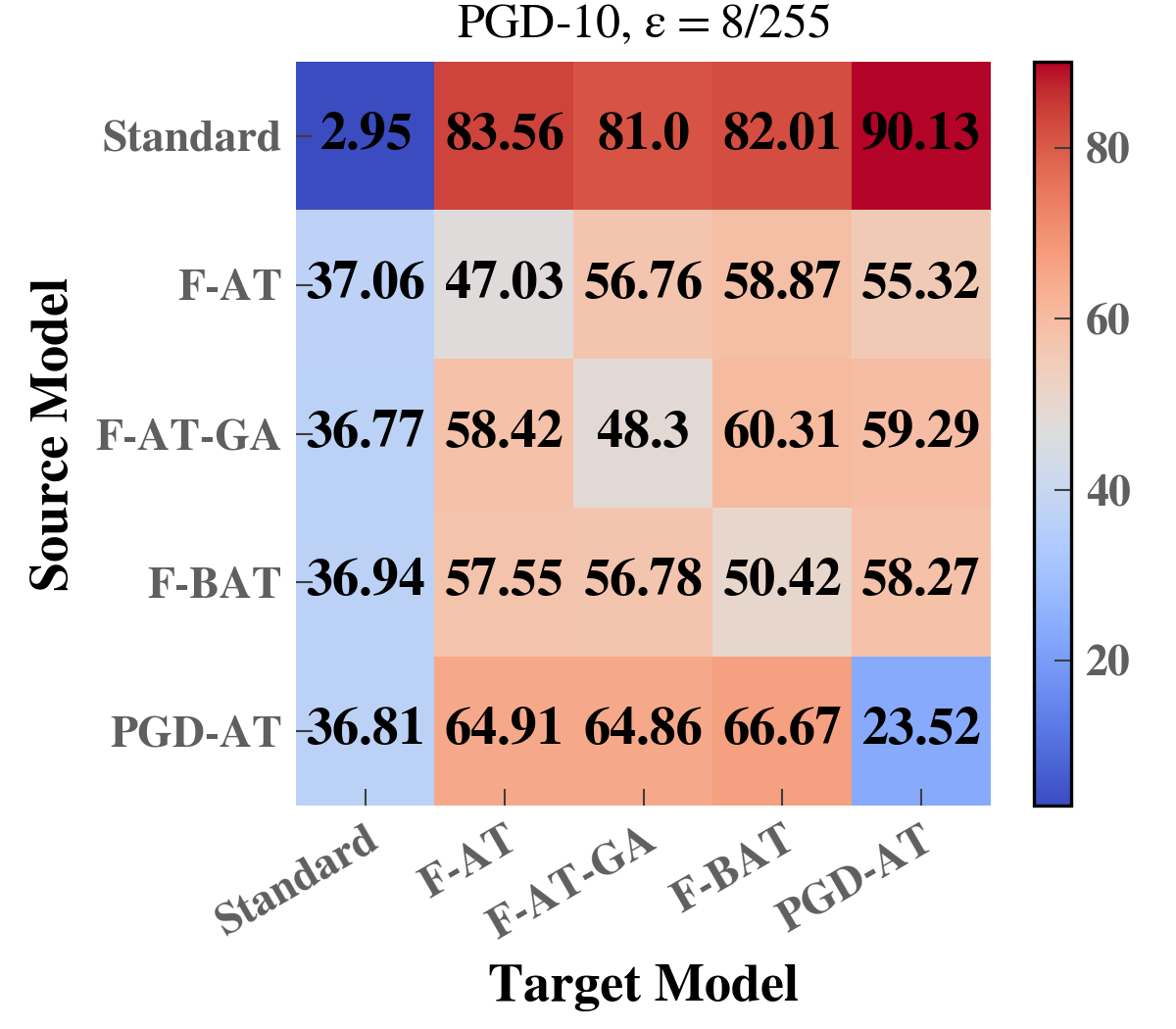} &\includegraphics[height=3.75cm,width=4.5cm,trim=0 0 0 0,clip]{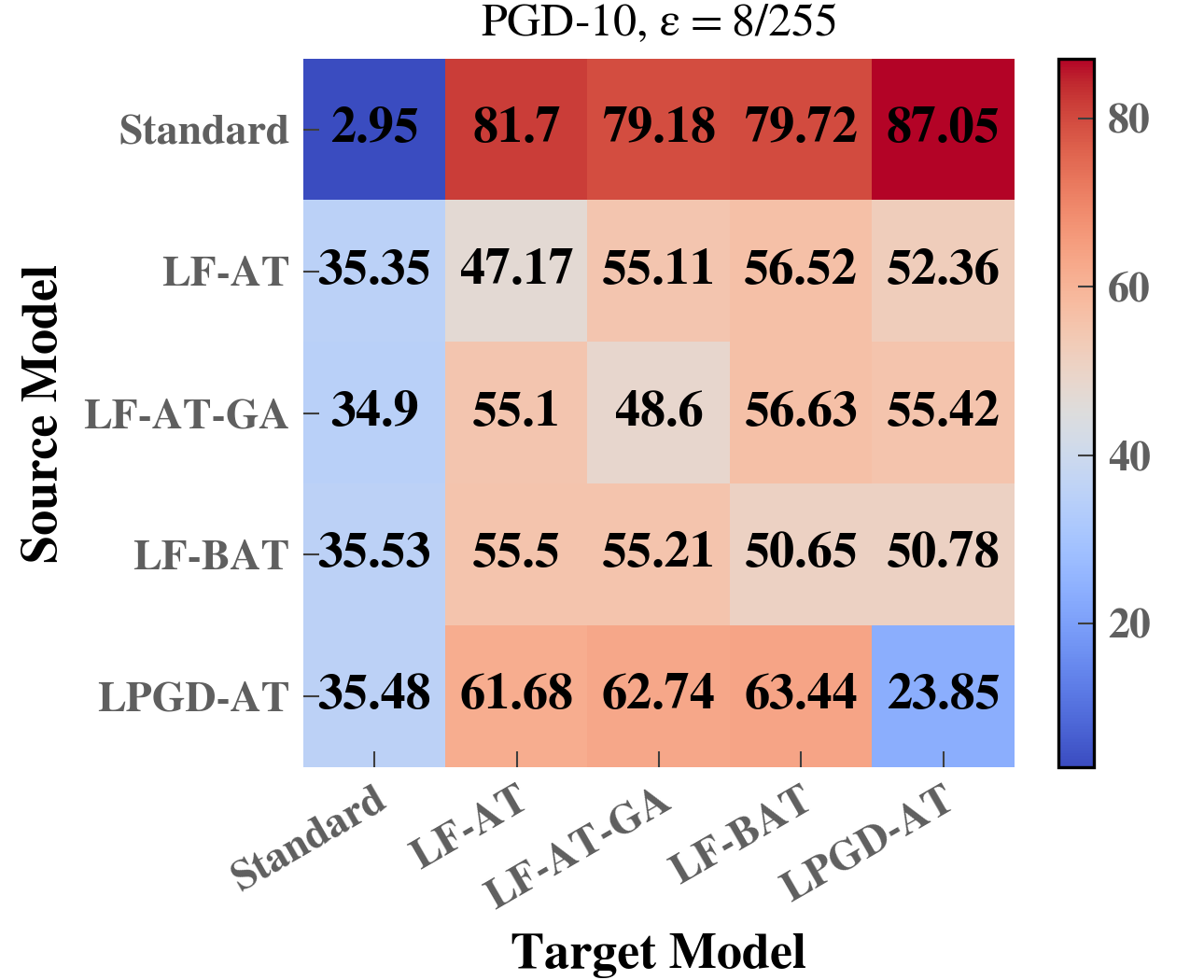} \\
			\footnotesize (a) SAT methods &\footnotesize (b) LAST (Ours)	\\ 
			%\specialrule{0em}{1pt}{3pt}
			%			\specialrule{0em}{0pt}{0pt}
			%			\footnotesize \quad\quad\quad\quad\quad(a) CIFAR10 \quad\quad &\footnotesize \quad\quad\quad(b) CIFAR100 \quad\quad \\
		\end{tabular}
	\end{center}
	\caption{We visualize the heatmap of four SAT methods including Fast-AT, Fast-AT-GA, Fast-BAT, PGD-AT (i.e., 2-step PGD-AT) and their improved version under transfer-based PGD-10 attack on CIFAR10 dataset. We use F-AT and LF-AT as abbreviations for Fast-AT and LFast-AT, and the other abbreviations are similar. It should be noted that the term ‘standard' refers to the test RA of models trained on clean data.}\label{fig:heat_map}
\end{figure}

\textbf{Validation with transfer-based attacks.} To validate the generalization performance of LAST, we explore the robustness of various SAT methods when subjected to black-box attacks. Specifically, we focus on the effectiveness of these methods against transfer-based black-box adversarial attacks, which are a practical concern in real-world scenarios where the attacker's model (source) may differ from the model being attacked (target). The heatmaps presented in Fig.~\ref{fig:heat_map} illustrate the RA of different SAT methods when they are attacked with adversarial perturbations generated by source models trained with Fast-AT, Fast-AT-GA, Fast-BAT, PGD-AT (which stands for two-step PGD-AT), and their respective improved versions under the LAST framework, all subjected to a transfer-based PGD-10 attack on the CIFAR10 dataset.

One of the key observation from the heatmaps is that adversarial attacks generated from source models trained with LAST are more difficult to defend against. This suggests that the LAST framework results in models that are more robust, and it implies that the models are less likely to be compromised by adversarial examples crafted based on different underlying models. Furthermore, the heatmaps also verify that both the original and improved SAT methods perform better under transfer-based attacks compared to white-box attacks.

%
%We also conduct analysis about the robustness of defense against black-box attacks. Practically, we plot the heatmaps of RA for different SAT methods against the transfer-based black-box adversarial attacks on CIFAR10 dataset under PGD-10 attack in Fig.~\ref{fig:heat_map}. Note that the source model are used to generate the adversarial perturbation to attack the target models. It is shown that adversarial attacks generated based on the source models trained by LAST are more difficult to defend, and both original methods and improved ones perform better under transfer-based attacks than white-box attacks.  

\begin{table}[!t]
	\centering
	\caption{Illustrating results of the RA for Fast-AT, Fast-AT-GA, Fast-BAT, PGD-AT and corresponding improved version under different perturbation sizes of PGD-10 and PGD-50 attack. The out-of-distribution samples are crafted with mixup technique on the CIFAR10 dataset.}
	\label{tab:AT_results8}
	\renewcommand\arraystretch{1.4}
	\setlength{\tabcolsep}{0.4mm}{
		\begin{tabular}{c|c|c|c|c}
			\toprule \hline \multicolumn{5}{c}{ CIFAR-10 dataset, PARN-18 backbone trained with $\boldsymbol{\epsilon}=8 / 255$} \\
			\hline \multirow{2}{*}{ Method } &  \multicolumn{2}{c|}{ PGD-10 (\%) }& \multicolumn{2}{c}{ PGD-50 (\%) }  \\
			& \multicolumn{1}{c|}{$\boldsymbol{\epsilon}=8 / 255$} &  \multicolumn{1}{c|}{$\boldsymbol{\epsilon}=16 / 255$}&  \multicolumn{1}{c|}{$\boldsymbol{\epsilon}=8 / 255$} &   \multicolumn{1}{c}{$\boldsymbol{\epsilon}=16/ 255$} \\
			\cmidrule{1-5} 
			Fast-AT  & $46.14$ & $14.46$ & $43.95$ & $9.25$ \\
			\rowcolor{gray!15} LFast-AT (Ours) & $\mathbf{46.37}$ & $\mathbf{14.77}$ & $\mathbf{44.79}$ & $\mathbf{10.33}$ \\
			\cmidrule{1-5}
			Fast-AT-GA & $48.27$ & $16.52$ & $45.76$ & $11.48$\\ 
			\rowcolor{gray!15}LFast-AT-GA  (Ours)& $\mathbf{48.65}$ & $\mathbf{17.26}$ & $\mathbf{46.57}$ & $\mathbf{12.73}$\\
			\cmidrule{1-5}
			Fast-BAT & $49.24$ & $18.41$ & $47.94$ & $13.64$\\
			\rowcolor{gray!15}LFast-BAT  (Ours)& $\mathbf{49.58}$ & $\mathbf{19.69}$ & $\mathbf{48.75}$& $\mathbf{15.48}$\\
			\cmidrule{1-5}
			PGD-AT & $49.53$ & $18.57$ & $47.88$ & $13.64$\\
			\rowcolor{gray!15}LPGD-AT  (Ours)& $\mathbf{51.87}$ & $\mathbf{21.90}$& $\mathbf{50.96}$& $\mathbf{17.70}$\\
			\hline \bottomrule
			
		\end{tabular}
	}
\end{table}

 \textbf{Results on out-of-distribution samples.} Last but not least, we validate the efficacy of the LAST framework when dealing with out-of-distribution (OOD) samples. To this end, we employed the PreActResNet-18 model, which was trained on the CIFAR10 dataset. To generate novel OOD samples, we applied the mixup technique on the validation subset of the CIFAR10 dataset. This approach allowed us to create a diverse set of samples that are not directly represented in the training data, thus simulating a more realistic scenario where models are exposed to data that deviates from the training distribution. In Tab~\ref{tab:AT_results8}, we compare the RA results for each method at two different perturbations sizes including $\boldsymbol{\epsilon} = 8/255$ and $\boldsymbol{\epsilon} = 16/255$.
  
It can be clearly observed that models trained by LAST consistently outperform these SAT methods in terms of robustness against adversarial attacks that utilize OOD samples. This is evidenced by the higher RA reported for SAT methods and their corresponding improved version, which suggests its generalization capabilities in defending against these adversarial examples that is distributionally different from the training data, a common challenge in real-world applications.

\section{Conclusion}~\label	{sec:conclusion}

In this research, we introduce LAST, an innovative adversarial defense framework designed to augment the SAT pipeline. This framework is distinguished by introducing the historical states of the target model as a proxy, which corrects the model's updates through a proxy-guided rule. This dynamic learned fast weights and calculated differential unit ensures that the model's learning trajectory is continuously refined, addressing the challenges of instability. We also incorporate a self-distillation defense objective, which requires no additional pretrained teacher models, thereby simplifying the training process and alleviating catastrophic overfitting. Our comprehensive experimental evaluations, which span across various datasets, backbone architectures, and attack types, consistently demonstrate LAST's ability to improve robust accuracy and training stability. The LAST framework, with its novel contributions to the field of AT, explores the potential of leveraging historical model knowledge to achieve state-of-the-art robustness.

\textbf{Broader Impact.} The proxy-guided update rule, a core innovation of the LAST framework, holds promise beyond the realm of adversarial robustness, and could be applied to a variety of scenarios where training instability is exacerbated by factors like poor data quality. The potential applications span across different domains, including but not limited to, image and speech recognition systems where noisy or biased training data are common. By enhancing the model's training stability and final performance, the LAST framework could lead to more reliable and robust AI systems, capable of generalizing better to real-world conditions. 

\ifCLASSOPTIONcaptionsoff
\newpage
\fi

\bibliographystyle{IEEEtran}
        \bibliography{TCSVT}

\newpage

% your own custom macro containing the $\backslash${\tt{includegraphics}} command to make things
 
\vspace{11pt}

\begin{IEEEbiography}[{\includegraphics[width=1in,height=1.25in,clip,keepaspectratio]{biography/LiuYaohua.pdf}}]{Yaohua Liu} received his M.S. degree  in Software Engineering from Dalian University of Technology, China in 2021. He is currently pursuing the Ph.D. degree in Software Engineering at Dalian University of Technology (DUT), Dalian, China. His research interests include computer vision, bilevel-optimization, adversarial attack and defense, and deep learning. 
\end{IEEEbiography}
\vspace{-1.2cm}
\begin{IEEEbiography}[{\includegraphics[width=1in,height=1.25in,clip,keepaspectratio]{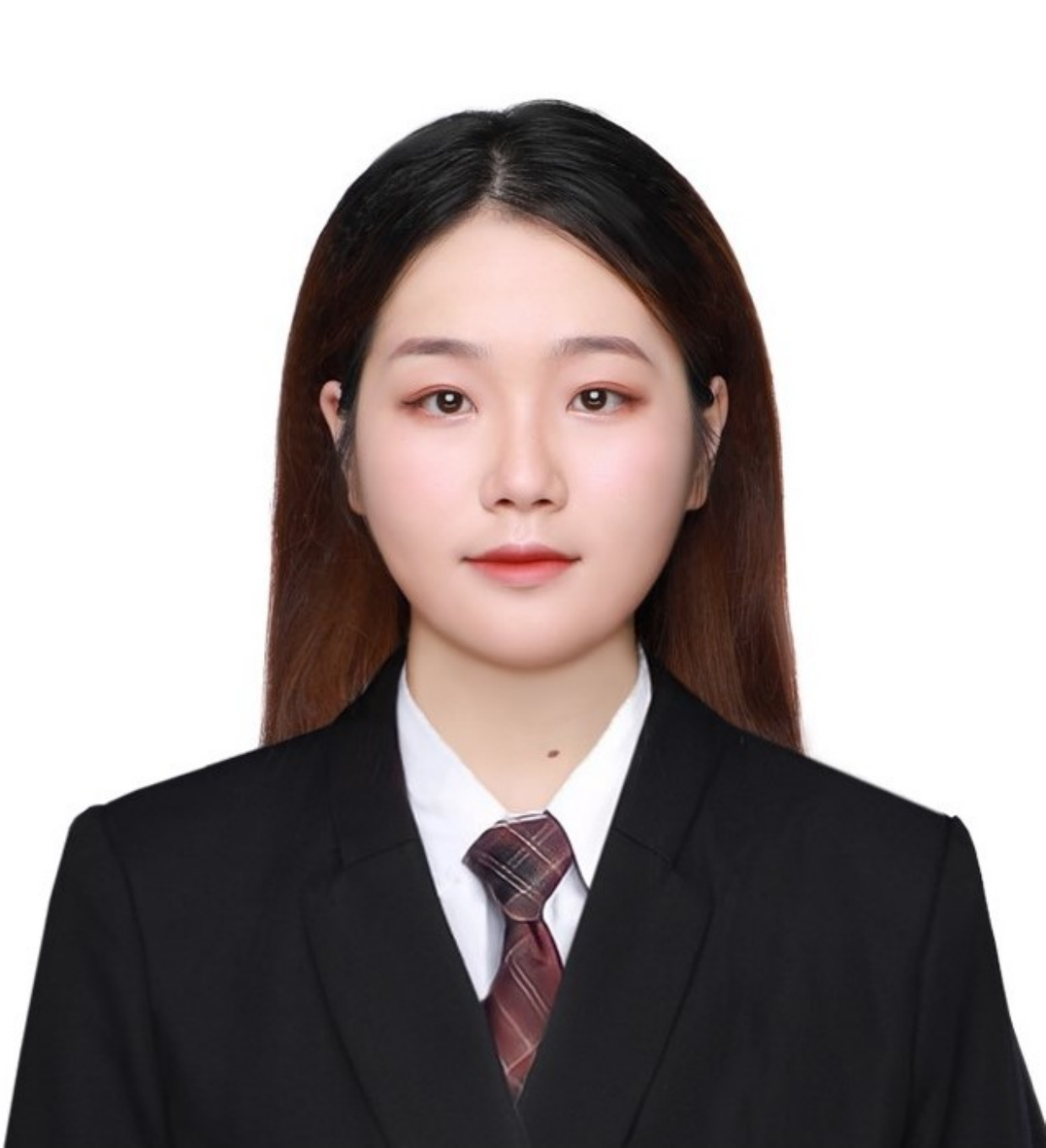}}]{Jiaxin Gao} received the B.S. degree in Applied Mathematics from Dalian University of Technology, China, in 2018. She is currently pursuing the PhD degree in software engineering at Dalian University of Technology, Dalian, China. She is with the Key Laboratory for Ubiquitous Network and Service Software of Liaoning Province, Dalian University of Technology, Dalian, China. Her research interests include computer vision, machine learning and optimization. 
\end{IEEEbiography}

\vspace{-1.3cm}
\begin{IEEEbiography}[{\includegraphics[width=1in,height=1.25in,clip,keepaspectratio]{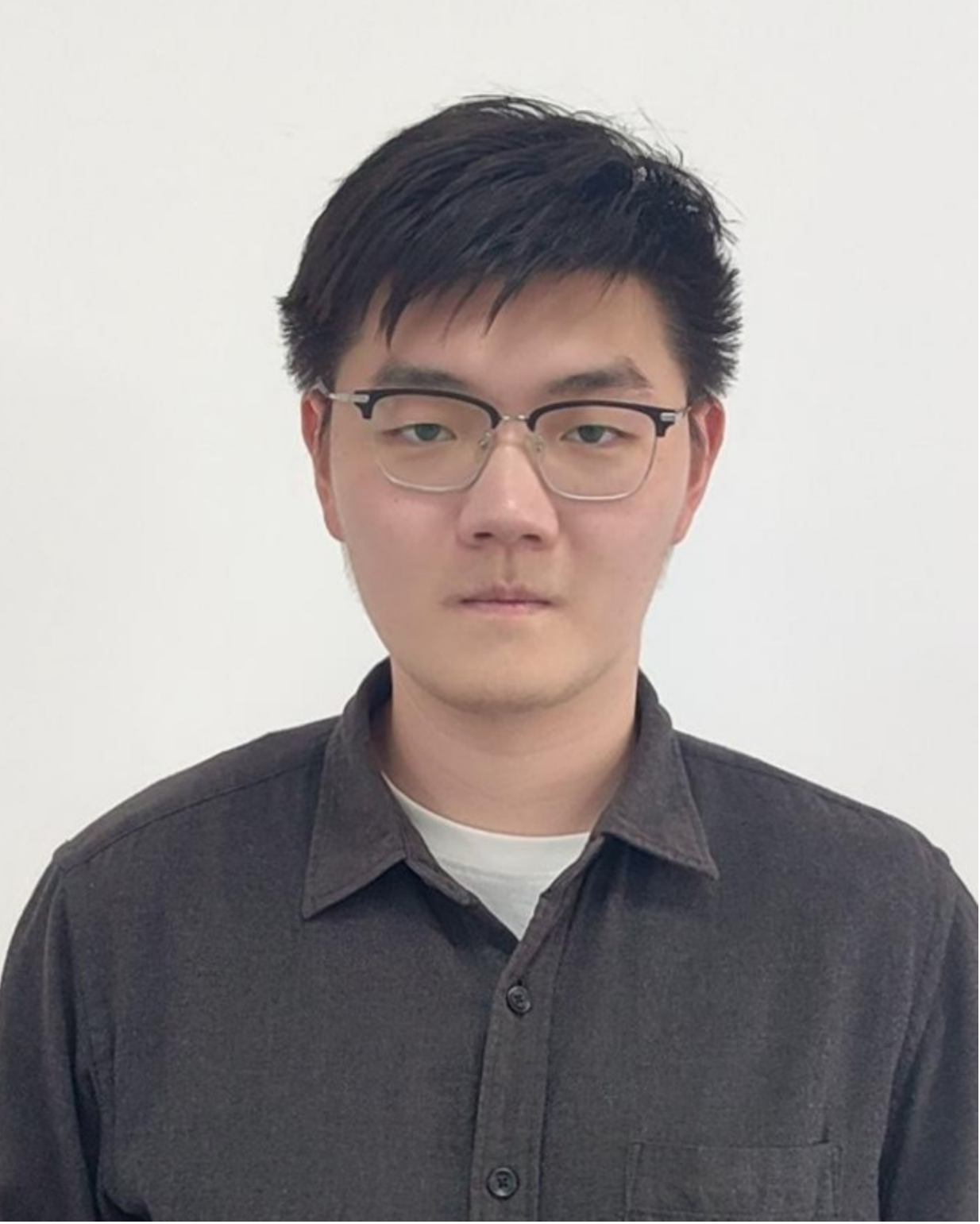}}]{Xianghao Jiao} received his B.S. degree in Software Engineering from Dalian University of Technology(DUT), Dalian, China, in 2021. He is currently pursuing a M.S. degree in Software Engineering at Dalian University of Technology(DUT), Dalian, China. His research interests include computer vision, adversarial defense, and bi-level optimization.
\end{IEEEbiography}
\vspace{-0.95cm}
\begin{IEEEbiography}[{\includegraphics[width=1in,height=1.25in,clip,keepaspectratio]{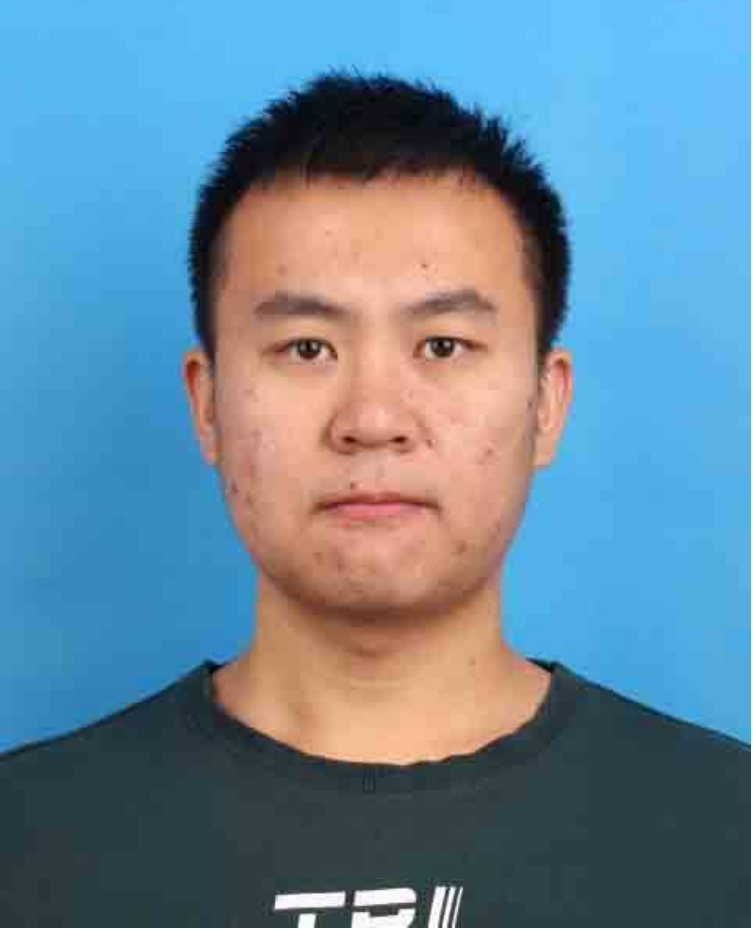}}]{Zhu Liu} received the B.S. degree in software	engineering from the Dalian University of Technology, Dalian, China, in 2019. He received his M.S. degree in Software Engineering at Dalian University of Technology, Dalian, China, in 2022. He is currently pursuing the Ph.D. degree in Software Engineering at Dalian University of Technology, Dalian, China. His research interests include optimization, image processing and fusion.
\end{IEEEbiography}
\vspace{-1.0cm}
\begin{IEEEbiography}[{\includegraphics[width=1in,height=1.25in,clip,keepaspectratio]{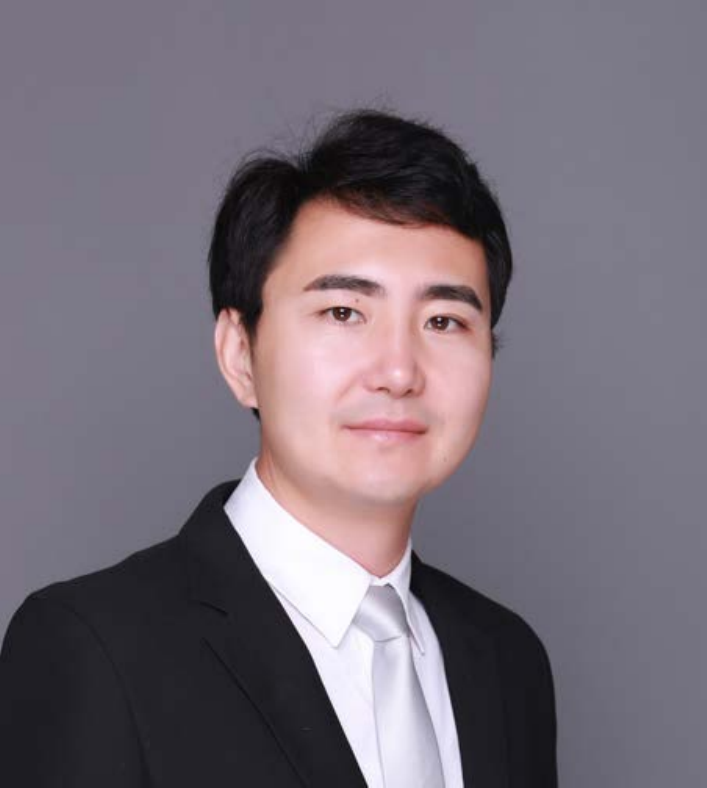}}] {Xin Fan} received the B.E. and Ph.D. degrees in information and communication engineering from Xian Jiaotong University, Xian, China, in 1998 and 2004, respectively. He was with Oklahoma State University, Stillwater, from 2006 to 2007, as a post-doctoral research fellow. He joined the School of Software, Dalian University of Technology, Dalian, China, in 2009. His current research interests include computational geometry and machine learning, and their applications to low-level image processing and DTI-MR image analysis.
\end{IEEEbiography}
\vspace{-1.2cm}
\begin{IEEEbiography}[{\includegraphics[width=1in,height=1.25in,clip,keepaspectratio]{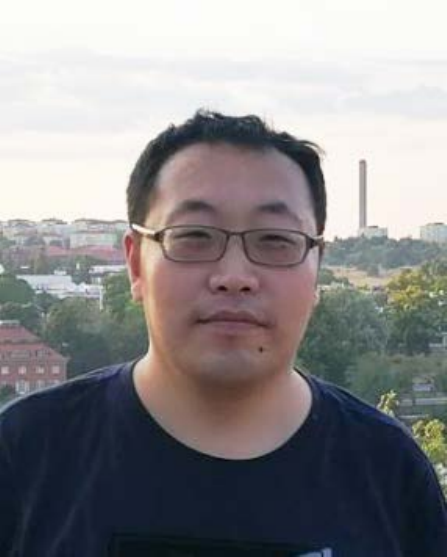}}]{Risheng Liu} received his B.Sc. (2007) and Ph.D. (2012) from Dalian University of Technology, China. From 2010 to 2012, he was doing research as joint Ph.D. in robotics institute at Carnegie Mellon University. From 2016 to 2018, He was doing research as Hong Kong Scholar at the Hong Kong Polytechnic University. He is currently a full professor of the School of Software Technology, Dalian University of Technology (DUT). He was awarded the ``Outstanding Youth Science Foundation" of the National Natural Science Foundation of China. He serves as Associate Editor for IEEE TCSVT, Pattern Recognition, Journal of Electronic Imaging, The Visual Computer, and IET Image Processing. His research interests include optimization, computer vision and deep learning.
\end{IEEEbiography}


\vfill

\end{document}